\documentclass[nohyperref]{article}
\usepackage[english]{babel}

\usepackage{microtype}
\usepackage{graphicx}
\usepackage{booktabs} % for professional tables
\usepackage{multirow}

\usepackage{hyperref}
\usepackage{url}
\usepackage{amsmath}
\usepackage{amsfonts}
\usepackage{amssymb}
\usepackage{amsthm}
\usepackage{dsfont}
\usepackage[table]{xcolor}
\usepackage{soul}
\usepackage{makecell}

\usepackage{thmtools} 
\usepackage{thm-restate}

\usepackage{graphicx}
\usepackage{caption}
\usepackage{subcaption}
\usepackage{bm}
\usepackage{breqn}
\usepackage{booktabs}

\usepackage{pifont}

\usepackage{tikz}
\usetikzlibrary{bayesnet}
\usepackage{adjustbox}

\newtheorem{proposition}{Proposition}

\usepackage[ruled, vlined, linesnumbered, noend, algo2e]{algorithm2e}
\DontPrintSemicolon
\SetKwInput{KwInput}{Input}

\newcommand{\node}{N}

\newcommand{\graph}{G}

\newcommand{\vars}{\bm{X}}

\newcommand{\dataset}{\mathcal{D}}

\newcommand{\order}{\sigma}
\newcommand{\orderless}{<}

\newcommand{\isdag}[1]{\text{DAG}(#1)}

\newcommand{\prior}{p_{pr}}
\newcommand{\likelihood}{p_{lh}}
\newcommand{\posterior}{p}

\newcommand{\weights}{B}
\newcommand{\noisevariance}{\Sigma}
\newcommand{\bias}{\bm{b}}
\newcommand{\causaleffect}[2]{E_{#1#2}}
\newcommand{\bayesiance}[2]{\text{BCE}(#1, #2)}

\newcommand{\spnweight}{\phi}

\newcommand{\trust}{\textsc{Trust}}
\newcommand{\trustd}{\textsc{Trust-d}}
\newcommand{\trustg}{\textsc{Trust-g}}
\newcommand{\dibs}{\textsc{Dibs}}
\newcommand{\gadget}{\textsc{Gadget}}

\newcommand{\graphp}{p_{G}}
\newcommand{\graphpi}[1]{p_{G_#1}}

\newcommand{\oracle}{\mathcal{O}}
\newcommand{\bnvariables}{\{1, ..., d\}}

\newcommand{\hlc}[2][yellow]{{%
    \colorlet{foo}{#1}%
    \sethlcolor{foo}\hl{#2}}%
}

\usepackage[accepted]{icml2022}

\usepackage{authblk}

\icmltitlerunning{Tractable Uncertainty for Structure Learning}

\begin{document}

\twocolumn[
\icmltitle{Tractable Uncertainty for Structure Learning}

\icmlsetsymbol{equal}{*}

\begin{icmlauthorlist}
\icmlauthor{Benjie Wang}{ox}
\icmlauthor{Matthew Wicker}{ox}
\icmlauthor{Marta Kwiatkowska}{ox}
\end{icmlauthorlist}

\icmlaffiliation{ox}{Department of Computer Science, University of Oxford, Oxford, United Kingdom}

\icmlcorrespondingauthor{Benjie Wang}{benjie.wang@cs.ox.ac.uk}

\icmlkeywords{Machine Learning, ICML, Causality, Structure Learning, Tractable Models, Causal Discovery, Probabilistic Circuits}

\vskip 0.3in
]

\printAffiliationsAndNotice{}

\begin{abstract} 
    Bayesian structure learning allows one to capture uncertainty over the causal directed acyclic graph (DAG) responsible for generating
    given data.
    In this work, we present Tractable Uncertainty for STructure learning (\trust), a framework for approximate posterior inference that relies on probabilistic circuits as the representation of our posterior belief. In contrast to sample-based posterior approximations, our representation can capture a much richer space of DAGs, while also being able to tractably reason about the uncertainty through a range of useful inference queries.
    We empirically show how probabilistic circuits can be used as an augmented representation for structure learning methods, leading to improvement in both the quality of inferred structures and posterior uncertainty. 
    Experimental results on conditional query answering further demonstrate the practical utility of the representational capacity of \trust{}.
\end{abstract}

\section{Introduction}

Understanding the causal and probabilistic relationship between variables of underlying data-generating processes can be a vital step in many scientific inquiries. Such systems are often represented by causal Bayesian networks (BNs), probabilistic models with structure expressed using a directed acyclic graph (DAG). The basic task of structure learning is to identify the underlying BN from a set of observational data, which, if successful, can provide useful insights about the relationships between random variables and the effects of potential interventions. However, even under strong assumptions such as causal sufficiency and faithfulness, it is typically impossible to identify a single causal DAG from purely observational data. Further, while consistent methods exist for producing a point estimate DAG in the limit of infinite data \citep{chickering2002ges}, in practice, when data is scarce many BNs can fit the data well. It thus becomes vitally important to quantify the uncertainty over causal structures, particularly in safety-critical scenarios.

Bayesian methods for structure learning tackle this problem by defining a prior and likelihood over DAGs, such that the posterior distribution can be used to reason about the uncertainty surrounding the learned causal edges, for instance by performing Bayesian model averaging. Unfortunately, the super-exponential space of DAGs makes both representing and learning such a posterior extremely challenging. A major breakthrough was the introduction of order-based representations \citep{Friedman03Ordermcmc}, in which the state space is reduced to the space of topological orders.

Unfortunately, the number of possible orders is still factorial in the dimension $d$, making it infeasible to represent the posterior as a tabular distribution over orders. Approximate Bayesian structure learning methods have thus mostly sought to approximate the distribution using samples of DAGs or orders \citep{Lorch21Dibs,agrawal2018imap}. However, such sample-based representations have very limited coverage of the posterior, restricting the information they can provide. Consider, for instance, the problem of finding the most probable graph extension, given an arbitrary set of required edges. Given the super-exponential space, even a large sample may not contain even a \emph{single} order consistent with the given set of edges, making answering such a query impossible.

A natural question, therefore, is whether it is possible to more compactly represent distributions over orders (and thus DAGs) while retaining the ability to perform useful inference queries tractably (in the size of the representation). We answer in the affirmative, by proposing a novel representation, OrderSPNs, for distributions over orders and graphs. Under the assumption of order-modularity, we show that OrderSPNs form a natural and flexible approximation to the target distribution. The key component is the encoding of hierarchical conditional independencies into the form of a sum-product network (SPN) \citep{Poon11Spn}, a well-known type of tractable probabilistic circuit. 
Based on this, we develop an approximate Bayesian structure learning framework, \trust{}, for efficiently querying OrderSPNs and learning them from data. Empirical results corroborate the increased representational capacity and coverage of \trust{}, while also demonstrating improved performance compared to competing methods on standard metrics.
Our contributions are as follows:
\begin{itemize}
    \item We introduce a novel representation, OrderSPNs, for Bayesian structure learning based on sum-product networks. In particular, we exploit exact hierarchical conditional independencies present in order-modular distributions. This allows OrderSPNs to express distributions over a potentially exponentially larger set of orders relative to their size. 
    \item We show that OrderSPNs satisfy desirable properties that enable tractable and exact inference. In particular, we present methods for computation of a range of useful inference queries in the context of structure learning, including marginal and conditional edge probabilities, graph sampling, maximal probability graph completions, and pairwise causal effects. We further provide complexity results for these queries; notably, all take at most linear time in the size of the circuit.
    \item We demonstrate how our method, \trust{}, can be used to approximately learn a posterior over DAG structures given observational data. In particular, we utilize a two-step procedure, in which we (i) propose a structure for the SPN using a seed sampler; and (ii) optimize the parameters of the SPN in a variational inference scheme. Crucially, the tractable properties of the circuit enable the ELBO and its gradients to be computed \textit{exactly} without sampling.
\end{itemize}
\section{Related Work}

Bayesian approaches to structure learning infer a distribution over possible causal graphs. Such distributions can then be queried to extract useful information, such as estimating causal effects, which can aid investigators in understanding the domain, or to plan interventions \citep{castelletti2021bayesian, maathuis2010predicting, Viinikka20Scalable}. Unfortunately, due to the super-exponential space, exact Bayesian inference methods for structure learning do not scale beyond $d=20$ \citep{koivisto2004exact, Kovisto06Advances}. As a result, there has been much interest in approximate methods, most notably performing MCMC sampling over the space of DAGs \citep{madigan1995bayesian, giudici2003improving}. 
Notable works in this direction include those by \citet{Friedman03Ordermcmc}, who operate over the much smaller space and smoother posterior landscape of topological orders, \citet{tsamardinos2006max}, who reduce the state-space by considering conditional independence, and \citet{Kuipers18Efficient}, who reduce the per-step computational cost associated with scoring. 

Alternatively, some recent works have applied variational inference to the Bayesian structure learning problem, where an approximate distribution over graphs is obtained by 
optimizing over some variational family describing distributions over graphs. 
Unfortunately, existing representations are typically not very tractable; \citet{Annadani21Vcn, cundy2021bcd} utilize neural autoregressive and energy-based models respectively, while \citet{Lorch21Dibs} employ sample-based approximations and particle variational inference \citep{Liu16Svgd}. This presents significant challenges for gradient-based optimization, since the reparameterization trick is not applicable for the discrete space of graphs. Further, downstream inference queries can only be estimated approximately through sampling. In contrast, our proposed variational family based on tractable models makes optimization and inference exact and efficient.

Probabilistic circuits \citep{Choi20Probcirc} are a general class of tractable probabilistic models which represent distributions using computational graphs. The key advantage of circuits, compared to other probabilistic models such as Bayesian networks, VAEs \citep{kingma2013auto}, or GANs \citep{Goodfellow14Gan}, is their ability to perform tractable and exact inference, for instance, computing marginal probabilities. While the typical use case is to learn a distribution over a set of variables from data \citep{Gens13Learnspn, Rooshenas14Idspn}, in this work we consider learning a circuit to approximate a given (intractable) posterior distribution over the space of DAGs, thus requiring different structure and parameter learning routines.
\section{Background}

\subsection{Bayesian Structure Learning}

\paragraph{Bayesian Networks} A Bayesian network (BN) $\mathcal{N} = (G, \Theta)$ is a probabilistic model $p(\bm{X})$ over $d$ variables $\bm{X} = \{X_1, ... X_{d}\}$, specified using the directed acyclic graph (DAG) $G$, which encodes conditional independencies in the distribution $p$, and $\Theta$, which parameterizes the mechanisms (conditional probability distributions) constituting the Bayesian network. The conditional probabilities take the form $p(X_i|\text{pa}_{G}(X_i), \Theta_i)$, giving rise to the joint data distribution:
\begin{align*}
    p(\bm{X}|G, \Theta) = \prod_i p(X_i|\text{pa}_{G}(X_i), \Theta_i)
\end{align*}
where $\text{pa}_{G}(X)$ denotes the parents of $X$ in $G$. One of the most popular types of BN model is the linear Gaussian model, under which the distribution is given by the structural equation
$\bm{X} = \bm{X}\weights + \bm{\epsilon} $,
where $\weights \in \mathbb{R}^{d \times d}$ is a matrix of real weights parameterizing the mechanisms, and $\bm{\epsilon} \sim \mathcal{N}(\bias, \noisevariance)$ where $\bias \in \mathbb{R}^d$ and $\noisevariance \in \mathbb{R}_{\geq 0}^{d \times d}$ is a diagonal matrix of noise variances. In particular, for a given DAG $G$, we have $B_{ij} = 0$ for all $i, j$ such that $i$ is not a parent of $j$ in $G$.

Whereas Bayesian networks typically only express probabilistic (conditional independence) information, causal Bayesian networks \citep{spirtes2000causation, Pearl09Causality} are additionally imbued with a causal interpretation, where, intuitively, the directed edges in $G$ represent direct causation. More formally, causal BNs can predict the effect (change in joint distribution) of interventions in the system, where some mechanism is changed, for instance by setting a variable $X$ to some value $x$ independent of its parents.

\paragraph{Bayesian Structure Learning} \textit{Structure learning} \citep{Koller06Pgm,Glymour19Review} is the problem of learning the DAG $G$ of the (causal) Bayesian network responsible for generating some given data $\dataset$.
Typically, strong assumptions are required for structure learning; in this work, we make the common assumption of causal sufficiency, meaning that there are no latent (unobserved) confounders. Even given this assumption, it is often not possible to reliably infer the causal DAG, whether due to limited data, or non-identifiability within a Markov equivalence class. Instead of learning a single DAG, Bayesian approaches to structure learning express uncertainty over structures in a unified fashion, through defining a prior $\prior(\graph)$ and (marginal) likelihood $\likelihood(\dataset|\graph)$ over directed graphs $G$.

A common assumption is that the prior and likelihood scores are \textit{modular}, that is, they decompose into a product of terms for each mechanism $G_i$ of the graph, where $G_i$ specifies the set of parents of variable $i$ in $G$. In such cases, the overall posterior decomposes as:
\begin{align*}
\posterior_G(G|\dataset)  &\propto \mathds{1}_{\isdag{G}} \prior(G) \likelihood(\dataset|G) \\ &= \mathds{1}_{\isdag{G}} \prod_i \prior(G_i) \likelihood(\dataset_i|G_i)
\end{align*}
The acyclicity constraint $\mathds{1}_{\isdag{G}}$ induces correlations between different mechanisms and presents the key computational challenge for posterior inference. The prior and likelihood can be chosen based on knowledge about the domain; for example, for linear Gaussian models, we can employ the BGe score \citep{Kuipers14Bge}, a closed form expression for the marginal likelihood of a variable given its parent set (marginalizing over weights of the linear model). The prior is typically chosen to penalize larger parent sets.

\subsection{Sum-Product Networks}

Sum-product networks (SPN) are probabilistic circuits over a set of variables $\bm{V}$, represented using a rooted DAG consisting of three types of nodes: leaf, sum and product nodes. These nodes can each be viewed as representing a distribution over some subset of variables $\bm{W} \subseteq \bm{V}$, where the root node specifies an overall distribution $q_\spnweight(\bm{V})$. Each leaf node $L$ specifies an input distribution over some subset of variables $\bm{W} \subseteq \bm{V}$, which is assumed to be tractable.
Each product node $P$ multiplies the distributions given by its children, i.e., $P = \prod_{C_i \in ch(P)} C_i$, while each sum node is defined by a weighted sum of its children, i.e., $T = \sum_{C_i \in ch(S)} \spnweight_i C_i$. The weights $\spnweight_i$ for each sum node satisfy $\spnweight_i > 0, \sum \spnweight_i = 1$, and are referred to as the \textit{parameters} of the SPN. The \textit{scope} of a node $N$ denotes the set of variables $N$ specifies a distribution over, and can be defined recursively as follows. Each leaf node $N$ has scope $sc(N) = \{V\}$, where $V$ is the variable it specifies its distribution over, and each product or sum node $N$ has scope $sc(N) = \cup_{C \in ch(N)} sc(C)$.

SPNs provide a computationally convenient representation of probability distributions, enabling efficient and exact inference for many types of queries, given certain structural properties \citep{Poon11Spn, Peharz14Selective}:
\begin{itemize}
    \item A SPN is \textit{complete} if, for every sum node $T$, and any two children $C_1, C_2$ of $T$, it holds that $sc(C_1) = sc(C_2)$. In other words, all the children of $T$, and thus $T$ itself, have the same scope.
    \item  A SPN is \textit{decomposable} if, for each product node $P$, and any two children $C_1, C_2$ of $P$, it holds that $sc(C_1) \cap sc(C_2) = \emptyset$. In other words, the scope of $P$ is partitioned by its children. 
    \item A SPN is \textit{deterministic} if, for each sum node $T$, and any instantiation $\bm{w}$ of its scope $sc(T) = \bm{W}$, at most one of its children $C_i(\bm{w})$ evaluates to a non-zero probability.  
\end{itemize}
Given completeness and decomposability, marginal inference becomes tractable, that is, we can compute $q_\spnweight(\bm{W})$ for any $\bm{W} \subseteq \bm{V}$ in linear time in the number of edges of the SPN. Conditional probabilities can be computed as the ratio of two marginal probabilities. If the SPN additionally satisfies determinism, MPE inference, i.e., $\max_{\bm{v}: \bm{W} = \bm{w}} q_\spnweight(\bm{v})$, also becomes tractable \citep{Peharz17Latent}.

\section{Tractable Representations for Bayesian Structure Learning} \label{sec:tractable}

In this work, we consider Bayesian structure learning over the joint space of topological orders and DAGs, where each order $\order$ is a permutation of $\{1, ..., d\}$. Let $\order^{\orderless i}$ be the set of variables preceding variable $i$ in $\order$. We say that a parent set $G_i$ is consistent with an order $\order$ if $G_i \subseteq \order^{\orderless i}$,
and that graph $G$ is consistent if all of its parent sets are consistent (written $G \models \order$). It follows that any DAG is consistent with at least one order, and further any directed graph consistent with an order must be acyclic. Thus we can specify a joint distribution over orders and DAGs as follows:
\begin{align*}
    \posterior(\order, \graph|\dataset) &\propto \posterior_G(\graph|\dataset) \mathds{1}_{\graph \models \order} \\ &= \prior(G) \likelihood(\dataset|G) \prod_i \mathds{1}_{G_i \subseteq \order^{\orderless i}}
\end{align*}
 Notice that the marginal $\posterior(\graph|\dataset)$ is not the same as $\posterior_G(\graph|\dataset)$, as $\posterior$ will favour graphs which are consistent with more orders. This imparts a bias for learning with respect to $p_G$. On the other hand, the space of orders is much smaller than the space of DAGs, enabling more efficient exploration of the distribution \citep{Friedman03Ordermcmc}.  
 
 In the case where the prior and likelihood are modular, the resulting distribution $p$ is said to be \textit{order-modular}. In this case, $\graphp(G)$ factorizes as $\graphp(G) = \prod_i \graphpi{i}(G_i)$, giving:
 \begin{align} \label{eqn:order-mod}
     p(\order, G) \propto \graphp(G) \mathds{1}_{G \models \order} =  \prod_i \posterior_{G_i}(\graph_i) \mathds{1}_{G_i \subseteq \order^{\orderless i}}
 \end{align}
 where we have omitted the dependence on the dataset and write $p(\order, \graph)$ for the Bayesian posterior.

\begin{figure*}
    \centering
    
\begin{subfigure}{.57\textwidth}
\begin{tikzpicture}[scale=1.0,transform shape,wrap/.style={inner sep=0pt,
fit=#1,transform shape=false}] %
        \node[draw=none, text width=2cm, align=center] (P1) {+};
        \node[draw=none, below=0cm of P1, inner sep=0pt, fill=green, text width=1.5cm, align=center] (P1t1) {{\footnotesize$\emptyset$}};
        \node[draw=none, below=0.1cm of P1t1, inner sep=0pt, fill=cyan, text width=1.5cm, align=center] (P1t2) {{\footnotesize$\{1, 2, 3, 4\}$}};
        
        \node[draw=none, inner sep = 0.05cm, below=0.4cm of P1t2, xshift=-2.5cm] (M1) {$\times$};
        \node[draw=none, inner sep = 0.05cm, below=0.4cm of P1t2, xshift=-0cm] (M2) {$\times$};
        \node[draw=none, inner sep = 0.05cm, below=0.4cm of P1t2, xshift=2.5cm] (M3) {$\times$};

        \node[draw=none, below=0.3cm of M1, inner sep = 0.05cm, xshift=-0.4cm, text width=1cm, align=center] (P11) {+};
        \node[draw=none, below=0.3cm of M1, inner sep = 0.05cm, xshift=0.4cm, text width=1cm, align=center] (P12) {+};
        \node[draw=none, below=0.3cm of M2, inner sep = 0.05cm, xshift=-0.4cm, text width=1cm, align=center] (P21) {+};
        \node[draw=none, below=0.3cm of M2, inner sep = 0.05cm, xshift=0.4cm, text width=1cm, align=center] (P22) {+};
        \node[draw=none, below=0.3cm of M3, inner sep = 0.05cm, xshift=-0.4cm, text width=1cm, align=center] (P31) {+};
        \node[draw=none, below=0.3cm of M3, inner sep = 0.05cm, xshift=0.4cm, text width=1cm, align=center] (P32) {+};
        
        \node[draw=none, below=0.0cm of P11, minimum height=0.32cm, fill=green, inner sep = 0.0cm, xshift=0cm, text width=0.75cm, align=center] (P11t1) {{\footnotesize$\emptyset$}};
        \node[draw=none, below=0.1cm of P11t1, inner sep = 0cm, xshift=0cm, fill=cyan, text width=0.75cm, align=center] (P11t2) {{\footnotesize$\{1, 2\}$}};
        
        \node[draw=none, below=0.0cm of P12, minimum height=0.28cm, inner sep = 0.0cm, xshift=0cm,  fill=green, text width=0.75cm, align=center] (P12t1) {{\footnotesize$\{1, 2\}$}};
        \node[draw=none, below=0.1cm of P12t1, inner sep = 0cm, xshift=0cm, fill=cyan, text width=0.75cm, align=center] (P12t2) {{\footnotesize$\{3, 4\}$}};
        \node[draw=none, below=0.0cm of P21, minimum height=0.32cm, inner sep = 0cm, xshift=0cm,  fill=green, text width=0.75cm, align=center] (P21t1) {{\footnotesize$\emptyset$}};
        \node[draw=none, below=0.1cm of P21t1, inner sep = 0cm, xshift=0cm, fill=cyan, text width=0.75cm, align=center] (P21t2) {{\footnotesize$\{2, 3\}$}};
        \node[draw=none, below=0.0cm of P22, minimum height=0.28cm, inner sep = 0cm, xshift=0cm,  fill=green, text width=0.75cm, align=center] (P22t1) {{\footnotesize$\{2, 3\}$}};
        \node[draw=none, below=0.1cm of P22t1, inner sep = 0cm, xshift=0cm, fill=cyan, text width=0.75cm, align=center] (P22t2) {{\footnotesize$\{1, 4\}$}};
        \node[draw=none, below=0.0cm of P31, minimum height=0.32cm, inner sep = 0cm, xshift=0cm, fill=green, text width=0.75cm, align=center] (P31t1) {{\footnotesize$\emptyset$}};
        \node[draw=none, below=0.1cm of P31t1, inner sep = 0cm, xshift=0cm, fill=cyan, text width=0.75cm,align=center] (P31t2) {{\footnotesize$\{1, 4\}$}};
        \node[draw=none, below=0.0cm of P32, minimum height=0.28cm, inner sep = 0cm, xshift=0cm, fill=green, text width=0.75cm, align=center] (P32t1) {{\footnotesize$\{1, 4\}$}};
        \node[draw=none, below=0.1cm of P32t1, inner sep = 0cm, xshift=0cm, fill=cyan, text width=0.75cm, align=center] (P32t2) {{\footnotesize$\{2, 3\}$}};
        
        %\node[draw=none, below=0.0cm of P32, inner sep = 0.05cm, xshift=0.0cm, text width=1cm, align=center] (GG) {{\footnotesize$\{2, 3\}$\\$\{1, 4\}$}};
        
        \node[draw=none, inner sep = 0.05cm, below=0.5cm of P11t2, xshift=-2.5cm] (M7) {$\times$};
        \node[draw=none, inner sep = 0.05cm, below=0.5cm of P11t2, xshift=-0.5cm] (M8) {$\times$};
        
        \node[draw=none, inner sep = 0.05cm, below=0.5cm of P12t2, xshift=0.5cm] (M9) {$\times$};
        \node[draw=none, inner sep = 0.05cm, below=0.5cm of P12t2, xshift=2.5cm] (M10) {$\times$};
        
        \node[draw=none, below=0.3cm of M7, inner sep = 0.05cm, xshift=-0.4cm, text width=1cm, align=center] (P71) {L};
        \node[draw=none, below=0.3cm of M7, inner sep = 0.05cm, xshift=0.4cm, text width=1cm, align=center] (P72) {L};
        \node[draw=none, below=0.3cm of M8, inner sep = 0.05cm, xshift=-0.4cm, text width=1cm, align=center] (P81) {L};
        \node[draw=none, below=0.3cm of M8, inner sep = 0.05cm, xshift=0.4cm, text width=1cm, align=center] (P82) {L};
        \node[draw=none, below=0.3cm of M9, inner sep = 0.05cm, xshift=-0.4cm, text width=1cm, align=center] (P91) {L};
        \node[draw=none, below=0.3cm of M9, inner sep = 0.05cm, xshift=0.4cm, text width=1cm, align=center] (P92) {L};
        \node[draw=none, below=0.3cm of M10, inner sep = 0.05cm, xshift=-0.4cm, text width=1cm, align=center] (P101) {L};
        \node[draw=none, below=0.3cm of M10, inner sep = 0.05cm, xshift=0.4cm, text width=1cm, align=center] (P102) {L};
        
        \node[draw=none, below=0.0cm of P71, inner sep = 0.0cm, xshift=0cm, minimum height=0.32cm, fill=green, text width=0.75cm, align=center] (P71t1) {{\footnotesize$\emptyset$}};
        \node[draw=none, below=0.1cm of P71t1, inner sep = 0cm, xshift=0cm, fill=cyan, text width=0.75cm, align=center] (P71t2) {{\footnotesize$\{1\}$}};
        \node[draw=none, below=0.0cm of P72, inner sep = 0.0cm, xshift=0cm, minimum height=0.28cm, fill=green, text width=0.75cm, align=center] (P72t1) {{\footnotesize$\{1\}$}};
        \node[draw=none, below=0.1cm of P72t1, inner sep = 0cm, xshift=0cm, fill=cyan, text width=0.75cm, align=center] (P72t2) {{\footnotesize$\{2\}$}};
        \node[draw=none, below=0.0cm of P81, inner sep = 0.0cm, xshift=0cm, minimum height=0.32cm, fill=green, text width=0.75cm, align=center] (P81t1) {{\footnotesize$\emptyset$}};
        \node[draw=none, below=0.1cm of P81t1, inner sep = 0cm, xshift=0cm, fill=cyan, text width=0.75cm, align=center] (P81t2) {{\footnotesize$\{2\}$}};
        \node[draw=none, below=0.0cm of P82, inner sep = 0.0cm, xshift=0cm, minimum height=0.28cm, fill=green, text width=0.75cm, align=center] (P82t1) {{\footnotesize$\{2\}$}};
        \node[draw=none, below=0.1cm of P82t1, inner sep = 0cm, xshift=0cm, fill=cyan, text width=0.75cm, align=center] (P82t2) {{\footnotesize$\{1\}$}};
        \node[draw=none, below=0.0cm of P91, inner sep = 0.0cm, xshift=-0.06cm, minimum height=0.28cm, fill=green, text width=0.75cm, align=center] (P91t1) {{\footnotesize$\{1, 2\}$}};
        \node[draw=none, below=0.1cm of P91t1, inner sep = 0cm, xshift=0cm, fill=cyan, text width=0.75cm, align=center] (P91t2) {{\footnotesize$\{3\}$}};
        \node[draw=none, below=0.0cm of P92, inner sep = 0.0cm, xshift=0cm, minimum height=0.28cm, fill=green, text width=1cm, align=center] (P92t1) {{\footnotesize$\{1, 2, 3\}$}};
        \node[draw=none, below=0.1cm of P92t1, inner sep = 0cm, xshift=0cm, fill=cyan, text width=0.75cm, align=center] (P92t2) {{\footnotesize$\{4\}$}};
        \node[draw=none, below=0.0cm of P101, inner sep = 0.0cm, xshift=-0.06cm, minimum height=0.28cm, fill=green, text width=0.75cm, align=center] (P101t1) {{\footnotesize$\{1, 2\}$}};
        \node[draw=none, below=0.1cm of P101t1, inner sep = 0cm, xshift=0cm, fill=cyan, text width=0.75cm, align=center] (P101t2) {{\footnotesize$\{4\}$}};
        \node[draw=none, below=0.0cm of P102, inner sep = 0.0cm, xshift=0cm, minimum height=0.28cm, fill=green, text width=1cm, align=center] (P102t1) {{\footnotesize$\{1, 2, 4\}$}};
        \node[draw=none, below=0.1cm of P102t1, inner sep = 0cm, xshift=0cm, fill=cyan, text width=0.75cm, align=center] (P102t2) {{\footnotesize$\{3\}$}};

        \draw[-] (P1t2) -- (M1) node [midway, above, inner sep = 0.02cm] {\scriptsize 0.4};;
        \draw[-] (P1t2) -- (M2) node [midway, left, inner sep = 0.02cm] {\scriptsize 0.15};;
        \draw[-] (P1t2) -- (M3) node [midway, above, inner sep = 0.02cm] {\scriptsize 0.45};;

        \draw[-] (M1) -- (P11);
        \draw[-] (M1) -- (P12);
        \draw[-] (M2) -- (P21);
        \draw[-] (M2) -- (P22);
        \draw[-] (M3) -- (P31);
        \draw[-] (M3) -- (P32);

        \draw[-] (P11t2) -- (M7) node [midway, left] {\scriptsize 0.7};
        \draw[-] (P11t2) -- (M8) node [midway, right] {\scriptsize 0.3};
        \draw[-] (P12t2) -- (M9) node [midway, left] {\scriptsize 0.5};
        \draw[-] (P12t2) -- (M10) node [midway, right] {\scriptsize 0.5};
        
        \draw[-] (M7) -- (P71);
        \draw[-] (M7) -- (P72);
        \draw[-] (M8) -- (P81);
        \draw[-] (M8) -- (P82);
        \draw[-] (M9) -- (P91);
        \draw[-] (M9) -- (P92);
        \draw[-] (M10) -- (P101);
        \draw[-] (M10) -- (P102);

   \end{tikzpicture}
   \caption{Regular OrderSPN, with expansion factors $\bm{K} = (3, 2)$. Each sum/leaf node is labelled with its associated $($\hlc[green]{$S_1$}, \hlc[cyan]{$S_2$}$)$. Only one expansion beyond the first level is shown for clarity.}
       
    \end{subfigure}
    \begin{subfigure}{.03\textwidth}
        \;
    \end{subfigure}
    \begin{subfigure}{0.31\textwidth}
\scalebox{0.77}{
\begin{tabular}{@{}lll@{}}
\toprule
$(S_1, S_2)$                               & \makecell{Example \\ Orders} & \makecell{Example\\ Graph} \\ \midrule
$($\hlc[green]{$\emptyset$}, \hlc[cyan]{$\{1, 2, 3, 4\}$}$)$ &  \makecell{\textbf{(1, 2, 4, 3)} \\ (2, 3, 1, 4) \\ (4, 1, 3, 2)}      &  
\adjustbox{valign=m}{
\begin{tikzpicture}[scale=0.85,transform shape,wrap/.style={inner sep=0pt,
fit=####1,transform shape=false}] %
        \node[obs] (1) {1} ; %
        \node[obs, xshift=2cm] (2) {2} ;
        \node[obs, below=of 1] (3) {3} ;
        \node[obs, below=of 2] (4) {4} ; %

        \draw[->] (1) -- (2);
        \draw[->] (4) -- (3);
        \draw[->] (2) -- (3);
        \draw[->] (1) -- (4);
        \draw[->] (2) -- (4);
   \end{tikzpicture}
   }
\\ \midrule
$($\hlc[green]{$\{1, 2\}$}, \hlc[cyan]{$\{3, 4\}$}$)$                       &  \makecell{(3, 4) \\ \textbf{(4, 3)}}      &   
\adjustbox{valign=m}{
\begin{tikzpicture}[scale=0.85,transform shape,wrap/.style={inner sep=0pt,
fit=####1,transform shape=false}] %
        \node[latent] (1) {1} ; %
        \node[latent, xshift=2cm] (2) {2} ;
        \node[obs, below=of 1] (3) {3} ;
        \node[obs, below=of 2] (4) {4} ; %

        \draw[->] (4) -- (3);
        \draw[->] (2) -- (3);
        \draw[->] (1) -- (4);
        \draw[->] (2) -- (4);
% 		\draw[thick, ->] (A) -- (R);
% 		\draw[thick, ->] (R) -- (M);
% 		\draw[thick, ->] (A) -- (S);
% 		\draw[thick, ->] (R) -- (AC);
% 		\draw[thick, ->] (M) -- (C);
% 		\draw[thick, ->] (CS) -- (S);
% 		\draw[thick, ->] (S) -- (AC);
% 		\draw[thick, ->] (AC) -- (MC);
% 		\draw[thick, ->] (C) -- (MC);
% 		\draw[thick, ->] (A) -- (MC);
   \end{tikzpicture}
  }
\\ \midrule
$($\hlc[green]{$\{1, 2\}$}, \hlc[cyan]{$\{4\}$}$)$                       &  \makecell{\textbf{(4)}}      &   
\adjustbox{valign=m}{
\begin{tikzpicture}[scale=0.85,transform shape,wrap/.style={inner sep=0pt,
fit=####1,transform shape=false}] %
        \node[latent] (1) {1} ; %
        \node[latent, xshift=2cm] (2) {2} ;
        \node[obs, below=of 2] (4) {4} ; %

        \draw[->] (1) -- (4);
        \draw[->] (2) -- (4);
% 		\draw[thick, ->] (A) -- (R);
% 		\draw[thick, ->] (R) -- (M);
% 		\draw[thick, ->] (A) -- (S);
% 		\draw[thick, ->] (R) -- (AC);
% 		\draw[thick, ->] (M) -- (C);
% 		\draw[thick, ->] (CS) -- (S);
% 		\draw[thick, ->] (S) -- (AC);
% 		\draw[thick, ->] (AC) -- (MC);
% 		\draw[thick, ->] (C) -- (MC);
% 		\draw[thick, ->] (A) -- (MC);
   \end{tikzpicture}
   }
\\ \bottomrule
\end{tabular}
}
\caption{Example orders and graphs for 3 sum/leaf nodes. Graphs only include parent sets of $S_2$ (filled) variables.}
    \end{subfigure}
    
    \caption{Example of regular OrderSPN for $d=4$. Best viewed in color.}
    \label{fig:spn_example}
  \end{figure*}
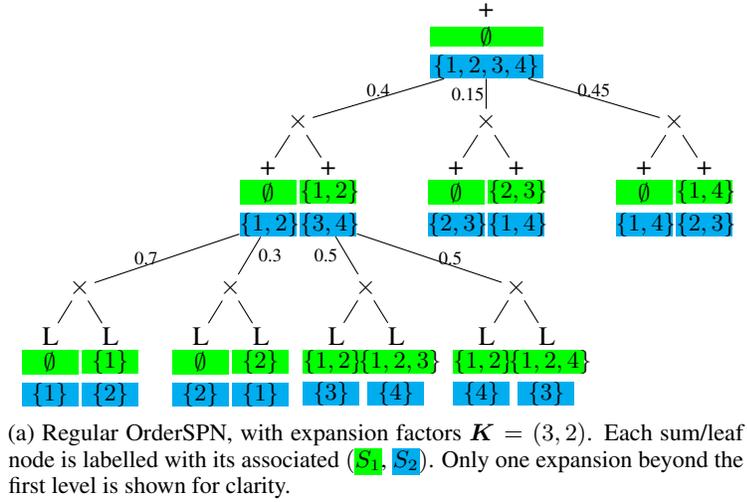
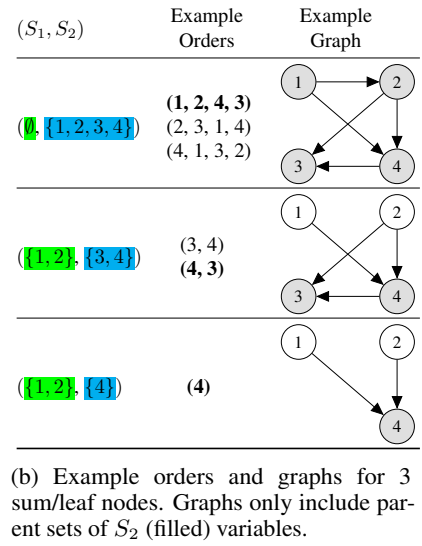

\subsection{Hierarchical CIs}

Unfortunately, the representation of the order-modular distribution in Equation \ref{eqn:order-mod} is not \textit{tractable}: we cannot easily sample from it, nor can we efficiently deduce, for instance, the marginal probability of a given edge. Our goal is thus to obtain a representation approximating this distribution which does possess tractable properties. The key idea is that by exploiting exact conditional independences (CIs) in the distribution, we can \textit{hierarchically} break the approximation of the original distribution into smaller subproblems.

To illustrate this, we first define some notation. 
Given any variable subset $S \subseteq \bnvariables$, let $\order_{S}$ denote an ordering (permutation) over variables in $S$, and $G_{S} \triangleq \{G_i: i \in S\}$ denote the set of parent sets for each variable $i$ in $S$.

Now, suppose we partition the set of BN variables $\{1, ..., d\}$ into two subsets $(S_1, S_2)$, and consider conditioning on the event that all variables in $S_1$ come before $S_2$ in the ordering, that is, the order partitions as $\order = (\order_{S_1}, \order_{S_2})$. In this case, the conditional distribution can be written as:
\begin{align*}
p(\order &, G |  \order = (\order_{S_1}, \order_{S_2})) 
\propto \prod_i \graphpi{i}(G_i) \mathds{1}_{G_i \subseteq (\order_{S_1}, \order_{S_2})^{\orderless i}}  \\
&= \prod_{i\in S_1} \graphpi{i}(G_i) \mathds{1}_{G_i \subseteq \order_{S_1}^{\orderless i}} 
\;\;\prod_{i\in S_2} \graphpi{i}(G_i) \mathds{1}_{G_i \subseteq S_1 \cup \order_{S_2}^{\orderless i} } 
\end{align*}
Notice that the distribution has factorized into two terms, which respectively include only $(\order_{S_1}, G_{S_1})$ and $(\order_{S_2}, G_{S_2})$. In fact, if we define the following (unnormalized) distribution over $(\order_{S_2}, G_{S_2})$:
$$\tilde{p}_{S_1, S_2}(\order_{S_2}, G_{S_2}) \triangleq \prod_{i\in S_2} \graphpi{i}(G_i) \mathds{1}_{G_i \subseteq S_1 \cup \order_{S_2}^{\orderless i} } $$
the previous factorization can be written as:
\begin{align*}
    \posterior(\order &, G| \order = (\order_{S_1}, \order_{S_2})) \\ &\propto \tilde{p}_{\emptyset, S_1}(\order_{S_1}, G_{S_1}) \tilde{p}_{S_1, S_2}(\order_{S_2}, G_{S_2})
\end{align*}

Thus, conditional on $\order = (\order_{S_1}, \order_{S_2})$, we have split the distribution over $(\order, G)$ into distributions over only $(\order_{S_1}, G_{S_1})$ and $(\order_{S_2}, G_{S_2})$ respectively.

Now, let us consider arbitrary disjoint subsets $S_1, S_2 \subseteq \bnvariables$. Then $\tilde{p}_{S_1, S_2}$ is a distribution over $\order_{S_2}, G_{S_2}$. We can apply a similar method to conditionally decompose $\tilde{p}_{S_1, S_2}(\order_{S_2}, G_{S_2})$ into distributions over $S_{21}, S_{22}$, where $S_{21}, S_{22}$ partition $S_2$, given by the following Proposition:
\begin{restatable}{proposition}{propDecomp}
\label{prop:decomp}
Let $p(\order, \graph) \propto \graphp(G) \mathds{1}_{G \models \order}$ be an order-modular distribution. Suppose that $S_1, S_2$ are any disjoint subsets of the variables $\bnvariables$, and let $(S_{21}, S_{22})$ be a partition of $S_2$. Then the following CI holds:
    \begin{align*}
        \tilde{p}_{S_1, S_2}(& \order_{S_2}, G_{S_2}|\order_{S_2} = (\order_{S_{21}}, \order_{S_{22}})) 
        \\ &\propto \tilde{p}_{S_1, S_{21}}(\order_{S_{21}}, G_{S_{21}}) \tilde{p}_{S_1 \cup S_{21}, S_{22}}(\order_{S_{22}}, G_{S_{22}})
    \end{align*}
\end{restatable}

These conditional independencies suggest an approximation strategy: select $K$ partitions $(S_1, S_2)$ of $\{1, ..., d\}$ to form the approximation and, conditional on a partition, then independently approximate the resulting distributions $\tilde{p}_{\emptyset, S_1}(\order_{S_1}, G_{S_1}), \tilde{p}_{S_1, S_2}(\order_{S_2}, G_{S_2})$, which are simpler problems of dimensions $|S_1|, |S_2|$, respectively. Using Proposition \ref{prop:decomp}, this can be done recursively, until we obtain distributions where $S_2$ is a singleton $\{i\}$, where:
\begin{align*}
%&\posterior_{S_1,\{\}} = 1 \\
    \tilde{p}_{S_1,\{i\}}(\order_{\{i\}}, G_i) &= \graphpi{i}(G_i) \mathds{1}_{G_i \subseteq S_1 \cup \order_{\{i\}}^{\orderless i}}  
    \\ &= \graphpi{i}(G_i) \mathds{1}_{G_i \subseteq S_1}
\end{align*}

\subsection{OrderSPNs}

The decomposition process can be viewed as a rooted tree, where we alternate between nodes that select partitions, and those which decompose the conditional distribution. This naturally induces a sum-product network structure, which we formalize in the following definition:
\begin{restatable}{definition}{defOSPN}
    An OrderSPN $q_{\spnweight}$ is a sum-product network over $(\order, G)$ with the following structure:
    
    \begin{itemize}
    
    \item Each leaf node $L$ is associated with $(S_1, \{i\})$, for some subset $S_1$ of $\bnvariables$ and $i \notin S_1$, and has scope $sc(L) = (\order_{\{i\}}, G_i)$. In addition, the leaf node distribution must have support only over graphs $G_i \subseteq S_1$.
     
     \item Each sum node $T$ is associated with two disjoint subsets $(S_1, S_2)$ of $\{1, ..., d\}$, where $|S_2| > 1$, and has scope $sc(T) = (\order_{S_2 }, G_{S_2})$. It has $K_T$ children and weights $\spnweight_{T, i}$ for $i = 1, ..., K_T$, where the i$^{\text{th}}$ child is a product node $P$ associated with $(S_1, S_{21, i}, S_{22, i})$ for some partition $(S_{21, i}, S_{22, i})$ of $S_2$.
    
    \item Each product node $P$ is associated with three disjoint subsets $(S_1, S_{21}, S_{22})$ of $\{1, ... d\}$, and has scope $sc(P) = (\order_{S_{21} \cup S_{22}}, G_{S_{21} \cup S_{22}})$, where $\order_{S_{21} \cup S_{22}}$ takes the form $(\order_{S_{21}},  \order_{S_{22}})$. It has two children, where the first child is associated with $(S_1, S_{21})$, and the second with $(S_1 \cup S_{21}, S_{22})$. These children are either sum-nodes or leaves.
    \end{itemize}
\end{restatable}

We can interpret each sum (or leaf) node $T$ associated with $(S_1, S_2)$ as representing a distribution over DAGs over variables $S_2$, where these variables can additionally have parents from among $S_1$. In other words, every sum node represents a (smaller) Bayesian structure learning problem over a set of variables $S_2$ and a set of potential confounders $S_1$.

In practice, we organize the SPN into alternating layers of sum and product nodes, starting with the root sum node. In the $j^{\text{th}}$ sum layer, we create a fixed number $K_j$ of children for each sum node $T$ in the layer. Further, for each child $i$ of each sum node $T$, we choose $(S_{21, i}, S_{22, i})$ such that $|S_{21, i}| = \lfloor \frac{|S_2|}{2} \rfloor, |S_{22, i}| = \lceil \frac{|S_2|}{2} \rceil$, and further require that the partitions are distinct for different children $i$ of $T$. Under these conditions, the OrderSPN will have $\lceil \log_{2}(d) \rceil$ sum (and product) layers. This ensures compactness of the representation, and enables efficient tensorized computation over layers.  We call such OrderSPNs \textit{regular}, and the associated list $\bm{K}$ of numbers of children for each layer are called the \textit{expansion factors}. An example of a regular OrderSPN is shown in Figure \ref{fig:spn_example}. At the top sum layer, we create a child for $K_1=3$ different partitions of $\{1, 2, 3, 4\}$ into equally sized subsets, each of which has an associated weight. Sum and product layers alternate until we reach the leaf nodes.

The leaf nodes $L$ represent distributions over some column of the graph: if $L$ is associated with $(S_1, i)$, then it expresses a distribution over the parents $G_i$ of variable $i$. The interpretation of $S_1$ is that this distribution should only have support over sets $G_i \subseteq S_1$. This restriction ensures that OrderSPNs are consistent, in the sense that they represent distributions over valid ($\order, \graph$) pairs (in particular, all graphs are acyclic): 
%In the following Proposition, we show the OrderSPNs only cover :

\begin{restatable}{proposition}{propOSPNacyc} \label{prop:OSPNacyc}
Let $q_\spnweight$ be an OrderSPN. Then, for all pairs $(\order, G)$ in the support of an OrderSPN, it holds that $G \models \order$. 
\end{restatable}

By design, (regular) OrderSPNs satisfy the standard SPN properties that make then an efficient representation for inference, which we show in the following Proposition. In the following sections, we use these properties for query computation, as well as for learning the SPN parameters.

\begin{restatable}{proposition}{propOSPNcdd} \label{prop:OSPNcdd}
Any OrderSPN is complete and decomposable, and regular OrderSPNs are additionally deterministic. 
\end{restatable}

\subsection{Leaf Distributions}

Given a leaf node associated with $(S_1, i)$, corresponding to a distribution on $G_i$, the only restriction imposed by the definition is that the distribution has support only on graphs $G_i \subseteq S_1$. Given that we are approximating an order-modular distribution $p(\order, \graph) \propto \prod_i p_{G_i}(G_i) \mathds{1}_{G_i \subseteq \order^{<i}}$, the natural choice of (unnormalized) leaf distributions is $
    p_{S_1,\{i\}}(\order_{\{i\}}, G_i)
    \propto \graphpi{i}(G_i) \mathds{1}_{G_i \subseteq S_1}
$; we provide formal justification for this choice in Proposition \ref{prop:leafjustification} in the Appendix. 

For tractable inference on the overall distribution over $(\order, \graph)$, we require that the leaf distributions can be computed tractably. In particular, we will be interested in three types of tasks: marginal/conditional inference, MPE inference, and (conditional) sampling.
To formalize this, let $a_{i, j}$ be a Boolean variable indicating whether $j \in G_i$, i.e., $j$ is a parent of $i$. Further, let $c_i$ be any logical conjunction of the corresponding positive or negative literals, i.e. $a_{i, j}$ or $\neg a_{i, j}$. For instance, $c_i = a_{i, 0} \wedge a_{i, 1} \wedge \neg a_{i, 2}$ represents the event that $0, 1$ are parents of $i$, but not $2$. Then, the task of marginal inference is to evaluate the probability $p_{S_1,\{i\}}(c_i = 1)$. Conditional inference is the task of $p_{S_1,\{i\}}(c_i = 1| c'_i = 1)$ for two conjunctions $c_i, c'_i$. MPE inference is $\max_{G_i} p_{S_1,\{i\}}(G_i|c_i = 1)$, while conditional sampling is the task of sampling from $p_{S_1,\{i\}}(G_i|c_i = 1)$. 

Unfortunately, these inference queries are intractable to compute without further assumptions. Following previous work \citep{Kuipers18Efficient, Viinikka20Scalable}, we limit the parents of each variable $i$ to a fixed set of candidates parents $C_i \subset \{1, ..., d\} \setminus \{i\}$, where the size of $|C_i|$ is chosen to be manageable (around 16). These candidate sets are chosen to maximize the coverage of the distribution mass. Given this, we approximate $\tilde{p}_{S_1,\{i\}}(G_i) \approx  p_{G_i}(G_i) \mathds{1}_{G_i \subseteq S_1 \cap C_i}$.

Given this approximation, we can then perform a precomputation taking $O(3^{|C_i|})$ time and space complexity, after which all of these queries require just a $O(1)$ lookup, except conditional sampling, which takes time $O(|C_i|)$. We provide further details of the method in Appendix \ref{apx:singlenode}.

\subsection{Tractable Queries} \label{sec:queries}

In general, being able to compute inference queries tractably individually for single-node distributions is not sufficient to perform inference on the overall distribution over DAGs. While previous works have tackled this problem by sampling single DAGs or orders, our key insight is that we can leverage the tractable properties of SPNs to hierarchically aggregate over order components. We now characterize the classes of queries that can be computed tractably for OrderSPNs, and their interpretation in the context of structure learning. Below we will write $q_\spnweight(G)$ to denote the marginal of $G$ in $q_\spnweight(\order, G)$, and denote the size of the SPN by $M$. 

\paragraph{Marginal and conditional inference} Let $c_i, c'_i$ be conjunctions over the graph column $G_i$. Then the \textit{marginal inference} problem is to compute $q_{\spnweight}(\bigwedge_{i=1}^{d} c_i)$. This can be interpreted as the probability of any arbitrary combination of edges (direct causal relations) simultaneously being present. It is well known that marginal/conditional inference queries can be computed exactly for a complete and decomposable SPN in linear time in the size of the circuit \citep{Poon11Spn}. Since marginal inference for the individual leaves requires just a constant-time lookup, the overall complexity is $O(M)$.

\paragraph{MPE inference}
MPE inference is the problem of finding the most likely instantiation of the variables, given some evidence. More precisely, we wish to compute $\max_{G} q_{\spnweight}(G|\bigwedge_{i=1}^{d} c_i)$, which allows us to, for instance, find the most likely extension of a partially specified DAG. This is tractable (in linear-time) provided that the SPN is \textit{deterministic} \citep{Choi17Determinism}. As MPE inference on the individual leaves requires just a constant-time lookup, the overall complexity is once again $O(M)$.

\paragraph{Sampling} 
Unconditional sampling from the OrderSPN is straightforward and efficient; we traverse the SPN top-down, choosing one child of each sum-node, and all children of each product-node, until we reach the leaf nodes, taking linear time in $d$. Coupled with the cost of sampling the leaf-node distributions, the overall complexity is $O(d\max_i|C_i|)$ per sample. Conditional sampling is more involved, and requires an $O(M)$ bottom-up computation which updates the SPN weights/probabilities according to the evidence, before sampling via top-down traversal \citep{Vergari19SPN}.

\begin{table}[t]
\centering
\begin{tabular}{@{}lll@{}}
\toprule
\textbf{Query}       & \textbf{Time} & \textbf{Space} \\ \midrule
Marginal/Conditional & $O(M)$              &   $O(M)$             \\
MPE                  & $O(M)$              &    $O(M)$            \\
Sampling             & $O(d^2)$              &  $O(d^2)$              \\
Conditional Sampling & $O(d^2 + M)$              &   $O(d^2 + M)$             \\
Pairwise Causal Effects       & $O(d^3M)$    &   $O(d^2M)$        \\ \bottomrule
\end{tabular}
\caption{Per-query complexity for OrderSPNs, for $d$ variables and OrderSPN of size $M$}
\label{tab:complexity}
\end{table}

\paragraph{Causal effects} We now turn to the computation of other types of queries specific to the Bayesian network setting. In the well-studied case of linear Gaussian Bayesian networks, one of the most important quantities for causal inference is pairwise causal effects, first studied by \citet{Wright34path} as the "method of path coefficients". In particular, for a given graph $G$ and weights $\weights$, the causal effect of $X_i$ on $X_j$, written $\causaleffect{i}{j}(\weights)$, is given by summing the weight of all directed paths from $i$ to $j$, where the weight of a path is given by the product of the weights of the edges along that path. Notice that, in cases where $i$ is not an ancestor of $j$, $\causaleffect{i}{j}(\weights) = 0$. 
Now, a priori, when we do not know the graph or weights, the causal effect is a random variable given by:
$$\causaleffect{i}{j} = \sum_{\pi \in F(\{1, ..., d\} \setminus \{i, j\})} \weights_{i,\pi_1} \weights_{\pi_{|\pi|}, j} \prod_{i=1}^{|\pi|-1} \weights_{\pi_i, \pi_{i+1}}$$
where $F(S)$ is the family of all ordered subsets of the variables $S$. From the Bayesian perspective, we would like to employ Bayesian model averaging to estimate the causal effect. This is given by:
$$\bayesiance{i}{j} \triangleq \mathbb{E}_{G \sim q_{\spnweight}(G)} [\mathbb{E}_{\weights \sim q(\weights|G)} [\causaleffect{i}{j}(\weights)] ] $$

While the other queries we have analyzed describe properties of the distribution $q_\spnweight(G)$ over causal graphs, here we are concerned with causal inference \textit{on the domain variables} $\vars$ themselves (induced by $q_\spnweight(G)$). This is a significant distinction for two reasons. Firstly, it is often the case that such quantities are of great practical interest, for instance, to estimate the effect of various types of treatments/interventions on patient outcomes. Secondly, even with full knowledge of the causal graph, inference in Bayesian networks is NP-hard in general, making it unclear how to efficiently transfer knowledge about the distribution over causal graphs to distributions over the domain variables. Fortunately, we find that in the case of linear Gaussian BNs, OrderSPNs possess the appropriate structure to compute Bayesian averaged causal effects efficiently:

\begin{restatable}{proposition}{propCE}
Given an OrderSPN representation $q_\spnweight$ of the distribution over DAGs, the matrix of all pairwise Bayesian averaged causal effects $\bayesiance{i}{j}$ with respect to $q_\spnweight$ can be computed in $O(d^3M)$  time and $O(d^2M)$ space, where $M$ is the size of the SPN. 
\end{restatable}

The factor of $O(d^3)$ is unsurprising and arises from the inference cost of causal effects in linear Gaussian BNs \citep{Koller06Pgm}; the significance is in the linear complexity in the size of the OrderSPN, given that $\bayesiance{i}{j}$ averages over potentially exponentially more DAGs. Intuitively, this is achieved by ``summing-out'' over different causal (directed) paths between variables $i, j$ at each node.  We provide more details and a formal proof in Appendix \ref{apx:queries}.
\section{Structure Learning with OrderSPNs} \label{sec:learning}

In this section we propose a framework for learning OrderSPNs from data. This consists of two components. Firstly, we learn a structure for the OrderSPN, which characterizes the support of the distribution. Then, we optimize the parameters of the SPN using a variational inference scheme.

\subsection{SPN structure learning} \label{sec:structurelearning}

We focus on regular OrderSPNs, where the topology of the SPN is fixed, but for each sum node $T$ in layer $j$ we must choose the $K_j$ partitions $(S_{21, i}, S_{22, i})$ of $S_2$. We define the problem as choosing an oracle $\oracle$ which takes as input some data $\dataset$, disjoint sets $S_1, S_2$, and a number of samples $K$, and returns $K$ partitions $(S_{21, i}, S_{22, i})$ of $S_2$. The goal of the oracle is to maximize coverage of the posterior distribution, i.e. the posterior mass of orders consistent with at least one of the sampled partitions. 
In practice, we can instantiate the oracle with any Bayesian structure learning method that can be modified to produce a DAG over $S_2$, which can additionally have parents from $S_1$. In particular, we adapt two recent Bayesian structure learners, \dibs{} \citep{Lorch21Dibs} and \gadget{} \citep{Viinikka20Scalable}. Given such a method, we can define the oracle by (i) taking $K$ samples of such DAGs; (ii) for each sample, choosing a random ordering consistent with the DAG; and (iii) splitting the ordering into a partition. Each partition $(S_{21, i}, S_{22, i})$ influences the support of the corresponding child of $T$, by restricting that $S_{21}$ comes before $S_{22}$ in the ordering. 

The proposed strategy involves calling the oracle $\oracle$ for each sum-node in the OrderSPN. This improves exploration of the space over the base structure learning method, by recursively exploring subspaces of DAGs over smaller subsets of variables $S_2 \subseteq \{1, ..., d\}$. However, it also appears to introduce a computational challenge since the number of sum-nodes in the SPN could be very large. Thankfully, though each successive sum-layer has $2K_j$ more sum-nodes, the dimension of the DAG space is halved, meaning that the oracle requires much less time. In practice, we ensure efficient implementation by the following methods: (i) we set a time budget appropriately for the oracle in each layer; (ii) for small dimensions $|S_2| \leq d'$ (chosen to be $4$), we avoid the constant-time overhead of each oracle run by instead explicitly enumerating over all partitions.

\subsection{Parameter Learning via Variational Inference}

Given a SPN structure, we now consider the task of learning the parameters of the SPN. We formulate this as a discrete variational inference (VI) problem. Given an unnormalized order-modular distribution $\tilde{p}(\order, G) = \graphp(G) \mathds{1}_{G \models \order}$, the evidence lower bound (ELBO) is given by:
$$\mathbb{E}_{q_\spnweight(\order, G)}[\log \tilde{p}(\order, G)] + H(q_\spnweight(\order, G))$$
where $H(q_\spnweight(\order, G)) = -\mathbb{E}_{q_\spnweight(\order, G)}[\log q_\spnweight(\order, G)]$ is the entropy of the OrderSPN $q$. The goal of VI is then to maximize the ELBO with respect to $\spnweight$. Typically, such discrete VI problems are difficult, since the ELBO requires computing (gradients of) the expectation using high-variance estimators such as REINFORCE (due to the discrete space, the reparameterization trick is not applicable). Fortunately, due to the tractable properties of the SPN, this is not an issue:
\begin{restatable}{proposition}{propELBO}\label{prop:elbo}
The ELBO and its gradients for any regular OrderSPN $q_\spnweight$ and order-modular distribution $p$ can be computed in linear time in the size of the SPN.
\end{restatable}
We provide the proof in Appendix \ref{apx:vi}, which is based on the corresponding result (Thm 1) from \citet{Shih20Vi} for deterministic SPNs. This allows us to learn the SPN parameters using gradient-based optimization. Further, due to the layered structure of regular OrderSPNs, we can leverage tensor learning frameworks and hardware acceleration.

\section{Experiments}
\begin{figure*}[tb]
\centering
\includegraphics[width=\linewidth]{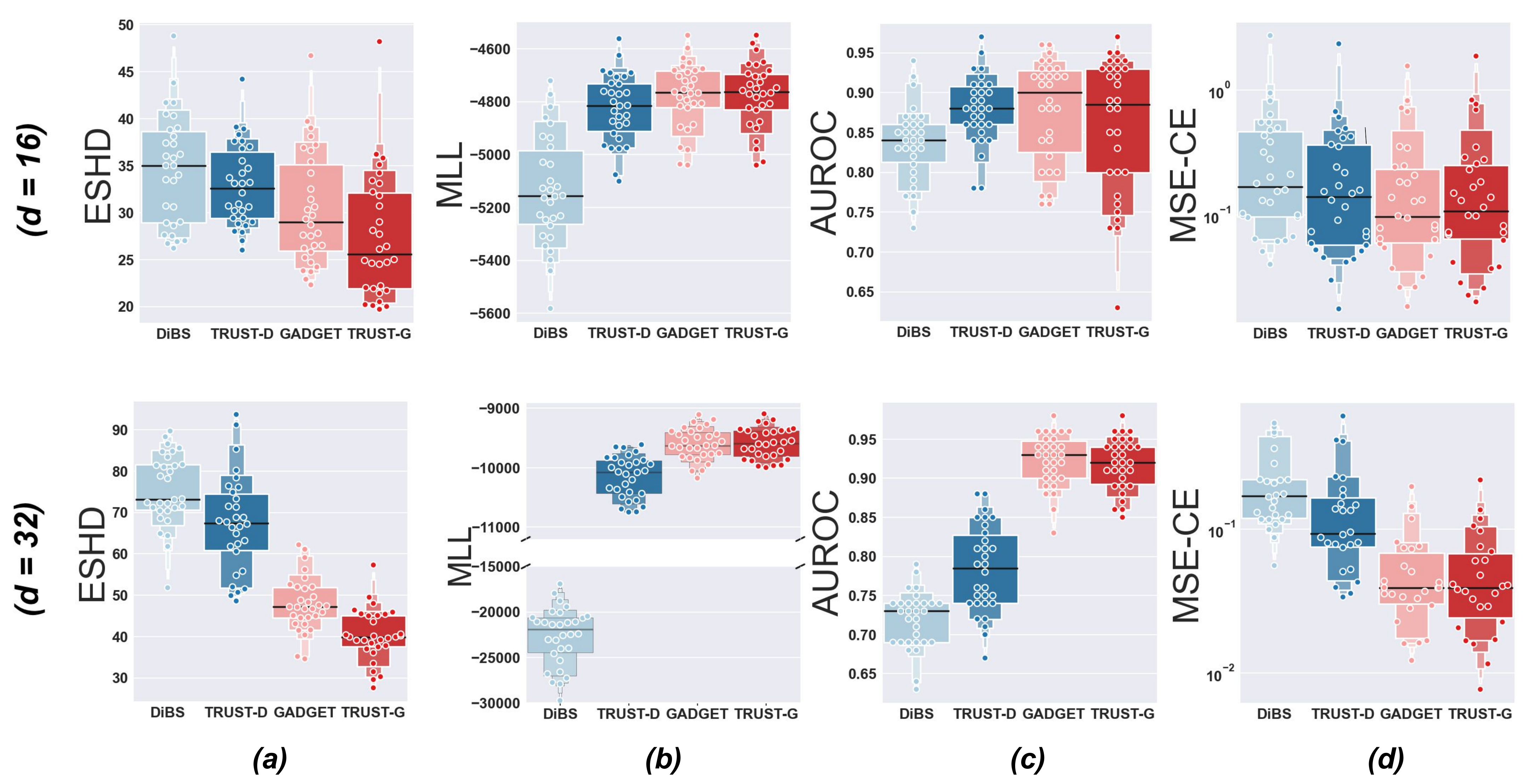}
\caption{Performance evaluation of the \trust{} framework. We find that across all metrics and for both dimensionalities that the \trust{} framework outperforms the seed method, in some instances considerably.  \textbf{Top Row:} Learning structures with $d=16$. \textbf{Bottom Row:} Learning structures with $d=32$. \textbf{(a)} Expected Structural Hamming Distance, lower is better. \textbf{(b)} Marginal Log Likelihood (higher is better). \textbf{(c)} Area Under the Receiver Operator Characteristic curve (higher is better). \textbf{(d)} MSE of Causal Effects (lower is better).}
\label{fig:PerformanceFigure}
\end{figure*}

In this section, we perform an empirical validation of the \trust{} framework.\footnote{Our implementation is available at \url{https://github.com/wangben88/trust}.} 
We implement two state-of-the-art Bayesian structure learning methods, (marginal) \dibs{} \citep{Lorch21Dibs} and \gadget{} \citep{Viinikka20Scalable}, and compare them with their \trust-enhanced counterparts, \trustd{} and \trustg{}, which use the respective method as the oracle for OrderSPN structure learning. 
Each inference method is applied to synthetic structure learning problems, where the ground truth causal structures are Erdős-R\'enyi random graphs with dimension $d \in \{16, 32\}$ and $2d$ expected edges, and the Bayesian network distribution is linear Gaussian. All methods tested employ the BGe marginal likelihood. For each experiment, a dataset $\mathcal{D}^{\text{train}}$ of $N=100$ datapoints is generated for each graph for inference.

\subsection{Learning Performance}

We begin by evaluating the quality of the inferred posterior $q(G|\mathcal{D})$ for each inference method, over a variety of standard metrics. In what follows, we use $G, B$ to denote the true graph/edge weights respectively, and $\mathcal{D}^{\text{test}}$ to denote a held-out dataset of $1000$ datapoints.

The \textit{expected structural Hamming distance} $\text{E-SHD}(q, G)$ measures the expected number of edge changes (SHD) between the essential graphs of $G$ and $G'$, where $G'$ is sampled from the posterior $q$:
\begin{align*}
        \text{E-SHD}(q, G) = \mathbb{E}_{G' \sim q}[\text{SHD(essential}(G'), \text{essential}(G))]
\end{align*}
The \textit{area under the receiver operating characteristic curve} $\text{AUROC}(q, G)$ for Bayesian structure learning \citep{Friedman03Ordermcmc} is computed using marginal edge probabilities $q(G'_{ij} = 1)$ for each potential edge $G'_{ij}$, while varying the confidence threshold to construct the ROC curve.

The \textit{marginal log-likelihood} $\text{MLL}(q, G, \mathcal{D}^{\text{test}})$ measures how well the posterior fits the held-out test data, using the BGe marginal likelihood $p$:
$$\text{MLL}(q, G, \mathcal{D}^{\text{test}})= \mathbb{E}_{G' \sim q}[\log p(\mathcal{D}^{\text{test}}|G')]$$

Finally, the \textit{mean-squared error of causal effects} $\text{MSE-CE}(q, B)$ measures the squared difference between the expected posterior causal effect $BCE_q(i, j)$, and the true causal effect $E_{ij}(B)$ (for variable pair $i, j$). This is then averaged over all (distinct) pairs $i, j$:
$$\text{MSE-CE}(q, B) = \frac{1}{d(d-1)} \sum_{i \neq j} |BCE_q(i, j) - E_{ij}(B)|^2$$

We show the results in Figure \ref{fig:PerformanceFigure} for all methods. \trustd{} and \trustg{} match or outperform their counterparts across all metrics, with especially strong performance on E-SHD, where \trustg{} is best by a clear margin for both $d= 16, 32$. 
Interestingly, we find that both the oracle methods used for SPN structure learning and the subsequent VI parameter learning are important for achieving the best posterior approximation; see Appendix \ref{sec:ablation} for further details.

\subsection{Coverage and Query Answering}

\begin{figure}[tb]
\centering
  \includegraphics[width=\linewidth]{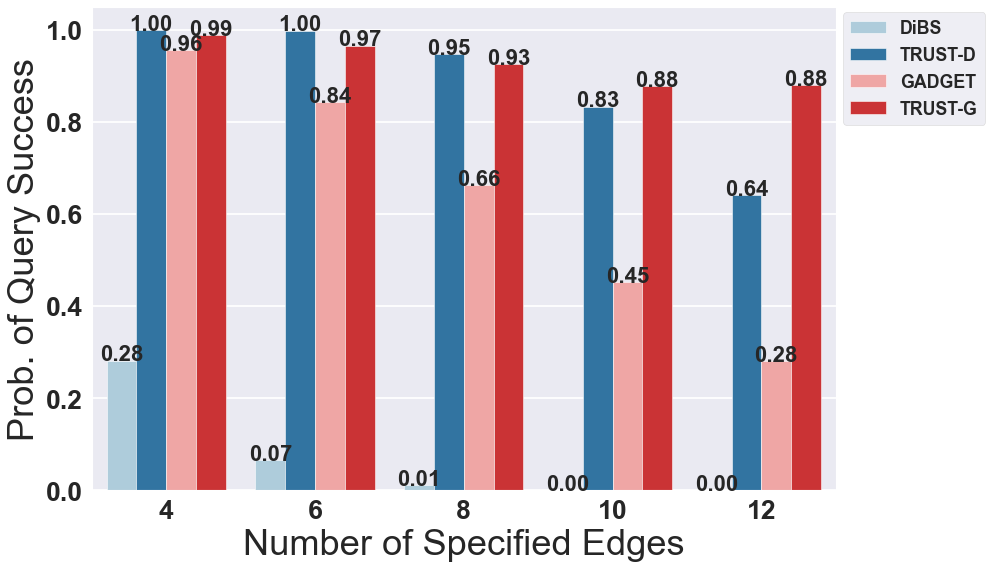}
  \caption{As we specify more edges in our query, the probability that 
  sample-based posteriors (\dibs{} and \gadget{}) have support over the queried edges drops. \trustd{} and \trustg{}, in contrast, maintain much greater coverage.}
  \label{fig:QueryProbs}
\end{figure}

We now compare the query answering capabilities of \trust{} to \dibs{} and \gadget{}, for $d=16$ networks. We set up the task by selecting $n$ edges randomly from the true graph, which we use to form the condition $\bigwedge_{i=1}^{d} c'_i$ in Section \ref{sec:queries} (requiring that all $n$ edges are present). A good representation of the posterior should consistently have posterior mass over this condition. For \dibs{} and \gadget{}, we obtain sample-based approximations $q$ of the posterior, for which we take $30$ and $10000$ samples respectively, as indicated by the respective papers and reference implementations. For \trustd{} and \trustg{} we directly perform the inference queries on the learned OrderSPN.

We begin by considering the marginal probability $q(\bigwedge_{i=1}^{d} c'_i)$. In Figure \ref{fig:QueryProbs}, we compute this over 30 different runs and 50 random edge selections for each run, for different values of $n$, and plot the proportion of times that the probability is non-zero. We see that, as $n$ increases, both methods based on \trust{} consistently outperform their counterparts. 
This demonstrates how \trust{} can be used to augment an oracle method to significantly improve the reliability of posterior coverage. This is particularly noteworthy for \dibs, whose coverage is otherwise limited by its quadratic time complexity in the number of samples.

From a practical perspective, this is especially important for conditional inference. In Table \ref{tab:condq}, we simulate a scenario where we obtain information on the true causal graph after learning. In particular, given $n = 4, 8, 16$ randomly specified edges from the true graph as a condition, we compute conditional probabilities for all unspecified (potential) edges. This can be viewed as "injecting" causal information, which, for instance, could permit distinguishing between DAGs in the same Markov equivalence class where observational data would not suffice. To evaluate, we compute the AUROC given the computed probabilities for each edge. In the case where the representation $q$ has no probability over the condition, we simply take the overall AUROC for the unconditional distribution. Table \ref{tab:condq} shows mean and standard deviation for AUROC over 30 runs for each method. As the number of specified edges increases, we see that the performance of \gadget{} degrades despite the extra information, since the sample-based representation suffers from prohibitively high variance when estimating conditional probability. On the other hand, the greater coverage of \trustg{} ensures that we can take advantage of the extra information, improving the quality of inferences.

\begin{table}[t]
\centering
\begin{tabular}{@{}llll@{}}
\toprule
\multirow{2}{*}{\textbf{No. Edges}} & \multirow{2}{*}{\textbf{Method}} & \multicolumn{1}{c}{\multirow{2}{*}{\textbf{AUROC}}} \\
                                    &                                  & \multicolumn{1}{c}{}                               & \multicolumn{1}{c}{} \\ \midrule
4                                   & \gadget{}                           & $0.905 \pm 0.073$ \\
                                    & \trustg{}                          &  $0.903 \pm 0.057$ \\ \midrule
8                                   & \gadget{}                           &   $0.888 \pm 0.089$ \\
                                    & \trustg{}                          &    $0.933 \pm 0.048$ \\ \midrule
16                                  & \gadget{}                          &     $0.876 \pm 0.081$ \\
                                    & \trustg{}                          &   $0.957 \pm 0.077$  \\ \bottomrule
\end{tabular}
\caption{Quality of inference for conditional queries. Results show  
\trustg{} is significantly better at inferring conditional distributions, especially as the condition becomes more restrictive.}
\label{tab:condq}
\end{table}
\section{Conclusion}

We study the problem of tractable representations in Bayesian structure learning. Such representations are crucial for being able to effectively learn and reason about causal structures with uncertainty. In particular, we introduce OrderSPNs, a new approximate representation of distributions over orders and structures. 
We show that OrderSPNs enable tractable and exact inference over the representation for a variety of important classes of queries, including, remarkably, inference of causal effects for linear Gaussian networks. Our experimental results demonstrate that OrderSPNs can indeed improve upon the representations of state-of-the-art Bayesian structure learning methods, with greater posterior coverage and query answering capabilities.

Our findings illustrate the potential of using tractable probabilistic representations to represent distributions over causal hypotheses. 
We anticipate that such representations could be applicable to a variety of tasks in causal inference, such as designing optimal interventions \citep{Agrawal19Abcd}, though we leave the investigation of these to future work.
Also, while we have chosen to focus on order-modular distributions and SPNs, it is an interesting question whether other types of distributions and tractable representations could be implemented. For instance, probabilistic sentential decision diagrams \citep{Kisa14Psdd} are a type of probabilistic circuit that admit a wider range of queries than SPNs. Such representations could offer alternative tractability properties that make them suitable for differing applications.
\paragraph{Acknowledgements}

We thank the anonymous reviewers for their valuable feedback and suggestions. 
This project was funded by the ERC under the European Union’s Horizon 2020 research and innovation programme (FUN2MODEL, grant agreement No.834115).

\bibliography{citations}
\bibliographystyle{icml2022}

\appendix
\clearpage

\section*{Appendix}

\section{Proofs of OrderSPN results}

In this section, we provide further details on the properties of OrderSPNs. Firstly, we prove Proposition \ref{prop:decomp}, regarding decompositions for order-modular distributions. Then, we provide proofs for Propositions \ref{prop:OSPNacyc} and \ref{prop:OSPNcdd} from the main paper, which show that (regular) OrderSPNs are consistent over orders and graphs, and satisfy the required properties for efficient inference. Finally, we formally characterize the compactness of the OrderSPN representation in a new result.

\subsection{Hierarchical Decomposition}

Recall that in Section \ref{sec:tractable}, we showed that the distribution on orders and graphs $(\order, G)$ could be decomposed into a product of two distributions on $(\order_{S_1}, G_{S_1})$ and $(\order_{S_2}, G_{S_2})$ respectively, conditional on $\order = (\order_{S_1}, \order_{S_2})$, i.e. the event that all variables in $S_1$ come before $S_2$ in the ordering. We now prove the generalization of that result, which allows us to \textit{hierarchically} decompose the distribution, giving rise to the proposed OrderSPN structure.

\propDecomp*
\begin{proof}
    By definition, we have that $\tilde{p}_{S_1, S_2}(\order_{S_2}, G_{S_2}) = \prod_{i\in S_2} \graphpi{i}(G_i) \mathds{1}_{G_i \subseteq S_1 \cup \order_{S_2}^{\orderless i} }$. Conditioning on the event $\order_{S_2} = (\order_{S_{21}}, \order_{S_{22}})$, we have that:
    \begin{align*}
        \tilde{p}_{S_1, S_2}&(\order_{S_2}, G_{S_2}| \order_{S_2} = (\order_{S_{21}}, \order_{S_{22}}))  \\
&\propto \prod_{i\in S_2} \graphpi{i}(G_i) \mathds{1}_{G_i \subseteq S_1 \cup \order_{(S_{21}, S_{22})}^{\orderless i} } \\
&= \prod_{i\in S_{21}} \graphpi{i}(G_i) \mathds{1}_{G_i \subseteq S_1 \cup \order_{S_{21}}^{\orderless i}}  \\
& \;\;\;\;\;\;\; \prod_{i\in S_{22}} \graphpi{i}(G_i) \mathds{1}_{G_i \subseteq S_1 \cup S_{21} \cup \order_{S_{22}}^{\orderless i} }  \\
&= \tilde{p}_{S_1, S_{21}}(\order_{S_{21}}, G_{S_{21}}) \tilde{p}_{S_1 \cup S_{21}, S_{22}}(\order_{S_{22}}, G_{S_{22}}) 
    \end{align*}
    as required.
\end{proof}

\subsection{OrderSPN properties}
 
 \propOSPNacyc*
\begin{proof}
    %\benjiecomment{first part TBD}
    A complete subcircuit $C$ \citep{Chan06Mpe, Dennis15Structure} is obtained by traversing the circuit top-down and i) selecting one child of every sum-node; ii) selecting all children of every product-node; iii) selecting all leaf nodes reached. By removing (or equivalently, setting to 1) all sum-node weights, $C$ is itself an OrderSPN expressing a distribution over $(\order, G)$. The key point is that the order is determined in any complete subcircuit. At the leaf nodes, the orders $\order_{\{i\}}$ over singletons are trivially deterministic. At the product nodes in the subcircuit, the order is determined by the order specified by the first (left) and second (right) child. That is, for a product node $P$ in the subcircuit associated with $(S_1, S_{21}, S_{22})$, if the left child specifies an order $\order_{S_{21}}$ and the right child an order $\order_{S_{22}}$, then the order for $P$ is determined as $(\order_{S_{21}}, \order_{S_{22}})$. Finally, the sum nodes in the subcircuit only have one child, so the order is determined from its child. Let the uniquely determined order for subcircuit $C$ be denoted $\order^C$.
    
    Now, consider any path from the root node to a leaf node in the subcircuit. Label the sum nodes (and leaf node) reached $T_i$ for $i = 1, ..., m$ (for some $m$), associated with $(S_{1, i}, S_{2, i})$ respectively. 
    We will now show, by induction, that for each sum node $T_i$, it is the case that all variables in $S_{1, i}$ come before $S_{2, i}$ in the ordering $\order^C$. 
    \begin{itemize}
        \item The root is associated with $(S_{1, 1}, S_{2, 1}) = (\emptyset, \{1, ... d\})$, so the condition is trivially satisfied.
        \item Now, given node $T_i$ with $i < m$, by definition $T_i$ has a product node child $P_i$ such that $T_{i+1}$ is either the first or second child of $P_i$. Let $P_i$ be associated with $(S_{1, i}, S_{21, i}, S_{22, i})$. Then, (i) if $T_{i+1}$ is the first child of $P_i$, then $(S_{1, i+1}, S_{2, i+1}) = (S_{1, i}, S_{21, i})$, while (ii) if $T_{i+1}$ is the second child of $P_i$, then $(S_{1, i+1}, S_{2, i+1}) = (S_{1, i} \cup S_{21, i}, S_{22, i})$. Now, $\order_C$ has the property that all nodes in $S_{21, i}$ come before those in $S_{22, i}$. Given the inductive hypothesis that $S_{1, i}$ comes before $S_{2, i}$ in the ordering, in both cases (i) and (ii) we have that all nodes in $S_{1, i+1}$ come before nodes in $S_{2, i+1}$ in the ordering.
    \end{itemize}
    This means that, at any leaf node associated with some $(S_1, \{i\})$, it will be the case that $S_1$ comes before $i$ in $\order^C$. Since the leaf distribution only has support over graphs with $G_i \subseteq S_1$, it follows that all graphs $G$ in the support satisfy $G \models \order^C$.
    
    The overall distribution of the OrderSPN is given by a (weighted) sum of all complete subcircuits, so the result follows.

    % Such a subcircuit defines a distribution over $(\order, G)$. Let 
    % For the first part, let $(\order, G)$ be in the support of the OrderSPN. Then 

    % , and the semantics of OrderSPNs that define the distribution 
    
    % The overall distribution of the OrderSPN can be obtained by
    % %via an inductive argument. In particular, we will show that for any node $N$ in the SPN with scope $(\order_S, G_S)$ (for some subset $S$ of variables), it holds that any value $g_s$ of $G_s$ which contradicts the order $\order_S$ is not in the support of the distribution of $N$. By definition, leaf nodes have scope $(\order_\{i\}, G_i)$, where the order is over a singleton and so any distribution over $G_i$ is consistent. For a product node, it is the 
    % \benjiecomment{just need to fill in final details here}.
\end{proof}

\propOSPNcdd*
\begin{proof}
    Given any sum node $T$ in the OrderSPN, completeness follows since the $i^{\text{th}}$ product node has scope $(\order_{S_{21, i} \cup S_{22, i}}, G_{S_{21, i} \cup S_{22, i}}) = (\order_{S_2}, G_{S_2})$, as $S_{21, i}, S_{22, i}$ partitions $S_2$ by definition. Decomposability follows immediately from the scopes of the product nodes $P$ and their children, where the variables $(\order_{S_{21} \cup S_{22}}, G_{S_{21} \cup S_{22}}$ are split into sum (or leaf) nodes with scope $(\order_{S_{21}}, G_{S_{21}})$ and $(\order_{S_{22}}, G_{S_{22}})$, where $S_{21}, S_{22}$ are disjoint. Determinism holds for regular OrderSPNs since every sum node has children which split the order into different partitions, so that the children have distinct support over orders (in fact, the choice of child at each sum node can be viewed as determining the order).
\end{proof}

\subsection{Compactness of OrderSPNs}

When organized as a regular OrderSPN, we can further characterize the compactness of the representation. The following result shows that OrderSPNs can be \textit{exponentially} more compact than sample representations of orders:

\begin{restatable}{proposition}{propOSPNsize}
Given a regular OrderSPN $q_\spnweight$ over $d = 2^l$ variables, with $l$ sum (and product) layers and expansion factors $(K_0, ... K_{l-1})$ as above, then we have that:
\begin{itemize}
    \item The size (number of edges) of $q_\spnweight$ is given by:
    $\sum_{i = 1}^{l} (2^i + 2^{i-1}) \prod_{j < i} K_j$
    \item The size (number of orders) of the support of $q_\spnweight$ is given by:
    $\prod_{i = 0}^{l-1} K_i ^ {2^{i}}$
\end{itemize}
\end{restatable}
\begin{proof}
    Let $T_i$, $P_i$ be the $i^{\text{th}}$ layer of the OrderSPN, with $|T_i|, |P_i|$ nodes respectively, for $i = 0, ..., l - 1$. We will also write $T_l$ to denote the leaf layer following all of the other layers. Then, by definition, the nodes in the $l^{\text{th}}$ sum layer each have $K_j$ children. Thus, $|P_i| = K_i|T_i|$. Each product node has two children, so we have the relation $|T_{l+1}| = 2|P_l|$. Since the first sum layer $P_1$ consists of just a single root node, $|P_0| = 1$, and it can be easily checked that $|T_i| = 2^{i} \prod_{j < i} K_j$ and $|P_i| = 2^{i} \prod_{j < i+1} K_j$. Thus the total number of nodes is given by:
    \begin{align*}
    \sum_{i = 0}^{l} |T_i| + \sum_{i = 0}^{l-1} |L_i| &= \sum_{i = 0}^{l} 2^{i} \prod_{j < i} K_j + \sum_{i = 0}^{l-1} 2^{i} \prod_{j < i+1} K_j \\
    &= 1 + \sum_{i=1}^{l} (2^{i} + 2^{i-1}) \prod_{j < i} K_j
    \end{align*}
    Note that the structure of an OrderSPN takes the form of a tree, i.e., each node has a unique parent (except the root). Thus, the number of edges in the OrderSPN is equal to the number of nodes, excluding the root.
    Now, each node represents a distribution over orders and graphs restricted to some subset of variables. Let $N(T_i)$ denote the number of distinct orders in the support of the first node in $T_i$ (similar for $N(P_i)$). Notice that, since $d = 2^l$, all nodes in the layer $T_i$ have support over the same number of orders. Thus, we need only consider how the number of orders covered changes as we move through the layers. Firstly note that, for the leaf layer $T_l$, all nodes express distributions over $(\order_{\{i\}}, G_i)$ for some variable $i$. There is only one possible permutation over a singleton set, so $N(T_l) = \order_{\{i\}}$. Then, for any sum-node in $T_i$, by determinism, each child of $T_i$ has disjoint support, so it follows that $N(T_i) = K_i \times N(P_i)$. For any product-node in $P_i$, we have that the two children of $P_i$ express distributions over orders/permutations $\order_{S_{21}}, \order_{S_{22}}$, where $S_{21}, S_{22}$ are disjoint sets. Since each of the children have support over $N(T_{i+1})$ orders, the product node expressing a distribution over $\order_{S_{21} \cup S_{22}}$ has $N(P_i) = N(T_{i+1})^2$. It is worth comparing this to the corresponding relation $|T_{l+1}| = 2|P_l|$ above; the conditional independence asserted by the OrderSPN results in the compactness of the representation. We can now see (by induction) that $N(P_i) = \prod_{j = i+1}^{l - 1} K_j^{2^{j - i + 1}}$ and $N(T_i) = \prod_{j = i}^{l -1} K_j^{2^{j-i}}$, and so the root node has support size:
    \begin{align*}
        N(T_0) = \prod_{j = 0}^{l -1} K_j^{2^{j}}
    \end{align*}
\end{proof}

\section{Computation of leaf distributions} \label{apx:singlenode}

We now explain in detail how to perform marginal, conditional, MPE and sampling inference for leaf distributions. 

Recall that leaf distributions for variable $i$ in an OrderSPN are given by the following density:
$$p_{S_1, \{i\}}(G_i) = \frac{p_{G_i}(G_i) \mathds{1}_{G_i \subseteq S_1}}{\sum_{G_i \subseteq S_1} p_{G_i}(G_i)}$$
where $S_1 \subset C_i$ is the set of potential parents of variable $i$, and where we have explicitly included the normalizing constant. 

As the dimension $d$ increases, this is challenging to compute due to the (exponential) sum over subsets in the normalizing constant. Further, different leaves of the OrderSPNs will in general have different sets $S_1$, due to the conditions on variable ordering imposed by the SPN structure. 

Thus, following previous work \citep{Friedman03Ordermcmc,Kuipers18Efficient}, we globally limit the parents of variable $i$ to a \textit{candidate set} $C_i$. That is, for each leaf node for variable i with distribution $p_{S_1, \{i\}}(G_i)$, we replace $S_1$ with $S_1 \cap C_i$. While this inevitably restricts the coverage of the distribution over DAGs, we can choose the candidate parents $C_i$ in such a way as to preserve as much posterior mass as possible\footnote{\cite{Viinikka20Scalable} studied a number of different strategies for selecting these candidate parents; we use the \textit{Greedy} heuristic, which was found empirically to be most effective.}. As we will shortly see, this enables us to design precomputation schemes that then allow for inference queries on $p_{S_1, \{i\}}(G_i)$ to be answered efficiently for \textit{any} $S_1 \subseteq C_i$.

%In the context of order/partition-based MCMC, it has been previously proposed \citep{Friedman03Ordermcmc,Kuipers18Efficient} \benjiecomment{check references} to globally limit the parents of variable $i$ to a \textit{candidate set} $C_i$. This provides much more efficient inference.

%This allows for vastly more efficient computation of the densities/normalizing constants, while preserving coverage of the distribution. The former is achieved by applying shared precomputation.

%\cite{Viinikka20Scalable} studied a number of different strategies for selecting these candidate parents. 

%Building upon this, we propose some simple stratgies for computing marginal

%The rationale behind this is that we can compute the normalizing constants $\tau_i(U) \triangleq \sum_{G_i \in U} p_{G_i}(G_i)$ for all $U \subset C_i$ in $O($, which allow us to then evaluate the density $p_{S_1, \{i\}}(G_i)$ for any $S_1 \subseteq C_i$ in constant time. 

In contrast to previous work, we are interested not just in the densities/normalizing constants, but also more complex forms of inference. For this, we define $a_{i, j}$ to be the event that $j \in G_i$, i.e. $j$ is a parent of $i$. Then, the key component is to precompute the following function, previously proposed in the Appendix of \citet{Viinikka20Scalable} (for a different purpose):
$$f_i(A_i, A_i') = \sum_{G_i \models \left(\bigwedge_{j \in A_i} a_{i, j} \wedge \bigwedge_{j \in A_i'} \neg a_{i, j}\right)} p_{G_i}(G_i)$$
where $A_i, A_i'$ are disjoint subsets of $C_i$. Intuitively, this is the (unnormalized) probability that all variables in $A_i$ are parents of $i$, and all those in $A_i'$ are not parents of $i$. 

This function can be precomputed in time and space $O(3^{|C_i|})$ as follows. In the base case where $A_i, A_i'$ partition $C_i$, then we simply have:
$$f_i(A_i, A_i') = p_{G_i}(A_i)$$
since $A_i, A_i'$ fully specify the parents of $i$.

In any other case, we have the recurrence:
$$f_i(A_i, A_i') = f_i(A_i \cup \{b\}, A_i') + f_i(A_i, A_i' \cup \{b\})$$
for any $b \in C_i \setminus (A_i \cup A_i')$.  This can be seen from the definition of $f_i$; the RHS corresponds to conditioning on the cases where $b$ either is or is not a parent of $i$. Notice that, for each $A_i, A_i'$, we need just a constant-time addition; thus the overall complexity is given by the number of partitions of $C_i$ into three subsets, i.e. $O(3^{|C_i|})$.

Now, let $c_i$ be any conjunction of (positive or negative) literals of the atoms $\{a_{i, j}: j \in C_i\}$, i.e., a partial specification of which edges can and can't be transformed. We now propose methods for performing marginal, conditional, MPE and sampling inferences:

\begin{itemize}
    \item \textbf{Marginal/Conditional}: For distribution $p_{S_1, \{i\}}$, the marginal for formula $c_i$ is given by:
    \begin{align*}
    p_{S_1, \{i\}}(c_i = 1) &=  \frac{\sum_{G_i \models c_i} p_{G_i}(G_i) \mathds{1}_{G_i \subseteq S_1}}{\sum_{G_i \subseteq S_1} p_{G_i}(G_i)} \\
    &= \frac{\sum_{G_i \models \left(c_i \wedge \bigwedge_{j \in C_i \setminus S_1} \neg a_{i, j} \right)} p_{G_i}(G_i)}{\sum_{G_i \models \bigwedge_{j \in C_i \setminus S_1} \neg a_{i, j}} p_{G_i}(G_i)}
    \end{align*}
    where we have expressed the condition that $G_i \subseteq S_1$ as the logical formula $\bigwedge_{j \in C_i \setminus S_1} \neg a_{i, j}$. Notice that both the numerator and denominator are of the form of the precomputed $f_i$, so we can compute the marginal probability simply by two lookups, i.e. $O(1)$ per query.
    
    Any conditional probability $p_{S_1, \{i\}}(c_i = 1|c_i'=1)$ can be computed from marginals as $p_{S_1, \{i\}}(c_i = 1|c_i'=1) = \frac{p_{S_1, \{i\}}(c_i \wedge c_i'=1)}{p_{S_1, \{i\}}(c_i'=1)}$.
    \item \textbf{MPE}: For distribution $p_{S_1, \{i\}}$, the MPE for formula $c_i$ is given by:
    \begin{align*}
    p_{S_1, \{i\}}(c_i = 1) &=  \max_{G_i} p_{S_1, \{i\}}(G_i|c_i = 1) \\
    &= \max_{G_i \models c_i} p_{S_1, \{i\}}(G_i|c_i = 1) \\
    &= \frac{\max_{G_i \models c_i \wedge \bigwedge_{j \in C_i \setminus S_1} \neg a_{i, j}} p_{G_i}(G_i)}{\sum_{G_i \models c_i \wedge \bigwedge_{j \in C_i \setminus S_1} \neg a_{i, j}} p_{G_i}(G_i)} 
    \end{align*} 
    The maximum is over $G_i$ satisfying a logical conjunction, similarly to how $f_i$ expresses sums over $G_i$ satisfying logical conjunctions. Thus, we propose to precompute another function $f^{\text{max}}_{i}$, which is entirely similar to $f_i$ except that the recurrence is given by:
    \begin{align*}
        &f^{\text{max}}_{i}(A_i, A_i') \\
        &= \max(f^{\text{max}}_{i}(A_i \cup \{b\}, A_i'), f^{\text{max}}_{i}(A_i, A_i' \cup \{b\}))
    \end{align*}
    Analogously to $f_i$, $f^{\text{max}}_{i}$ computes the maximal probability $p_{G_i}(G_i)$ for all $G_i$ satisfying the logical formula. Thus, once this function is precomputed, we can compute any MPE query through a lookup of $f^{\text{max}}_{i}$ and a lookup of $f_i$, i.e. $O(1)$ per query.
    \item \textbf{Sampling}: Given the condition $c_i$, we would like to sample $G_i$ from $p_{S_1, \{i\}}(G_i| c_i = 1)$. Let $B \subseteq C_i$ contain the variables which $c_i$ does not specify (as either definitely being a parent, or definitely not being a parent). 
    
    Then, given any ordering $b_1, ... b_K$ of the elements of $B$, we can sample whether $b_k$ is present sequentially. When sampling $b_k$, let $d^{(k)}_i$ be a conjunction formula representing the sampling of $b_1, ... b_{k-1}$, e.g., $d_i = a_{i, b_1} \wedge \neg a_{i, b_2} \wedge ... \wedge \neg a_{i, b_{k-1}}$. Then we have:
    \begin{align*}
        p_{S_1, \{i\}}(a_{b_k} = 1&|d^{(k)}_i=1, c_i = 1) \\
        &= p_{S_1, \{i\}}(a_{b_k} = 1|d^{(k)}_i \wedge c_i = 1)
    \end{align*}
    This takes the form of a conditional probability, which we can compute in constant time. We must apply this operation $K = O(|C_i|)$ times, which leads to an overall complexity of $O(|C_i|)$ per sampling query.
\end{itemize}

% For MPE inference, we need to compute a similar function using the following recurrence, where we once again use $f(A, A')$ as initialization for $A, A'$ partitioning the atoms:

% $$g(A, A') = \max (g(A \cup {b}), g(A' \cup {b}))$$

% We propose some simple methods for computing marginal, conditional, MPE and sampling inference. 

\section{Causal Effect Computation} \label{apx:queries}

% \propOSPN*

% \begin{proof}
%     Completeness follows since each child product node has scope $(\order_{S_{21, i} \cup S_{22, i}}, G_{S_{21, i} \cup S_{22}}) = (\order_{S_2}, G_{S_2})$ Decomposability follows immediately from the scopes of the product nodes $P$ and their children, where the variables $(\order_{S_{21} \cup S_{22}}, G_{S_{21} \cup S_{22}}$ are split into sum (or leaf) nodes with scope $(\order_{S_{21}}, G_{S_{21}})$ and $(\order_{S_{22}}, G_{S_{22}})$. Determinism holds for regular OrderSPNs since every sum node has children which split the order into different partitions, so that the children have distinct support over orders.
    
%     For the second part, the number of nodes in $q_\phi$ is given by the sum, where each term $\sum_{i = 1}^{l} (2^i + 2^{i-1}) \prod_{j < i} K_j$ refers to the number of nodes in sum and product layer $i$. Moving from each product node to each sum node, the layer size increases by a factor of 2, while moving from a sum node to a product node, the layer size increases by a factor of $K_i$. The number of orders in the support can be derived by noting that the support size of each sum node is given by the sum of the support sizes of its children (by determinism), and the support size of each product node is given by the product of the support sizes of its children (by decomposability).
% \end{proof}

In this section, we show how to tractably compute Bayesian averaged causal effects with respect to OrderSPN representations. The computation of BCE differs from the other queries, as $\causaleffect{i}{j}$ involves terms which are not localized to a leaf-node distribution; thus standard SPN inference routines are not applicable. Nonetheless, we find that it is possible to compute $\bayesiance{i}{j}$ for all $i, j$ \emph{exactly} with respect to the probabilistic circuit representation over orders and graphs.

%for query answering on (regular) OrderSPNs.\benjiecomment{TBC}

\propCE*
\begin{proof}
    %Firstly, given a graph $G$, we have that under the distribution $q(\weights|G)$.

    Recall that all nodes $t$ in the SPN can be associated with the variable subsets $(S_1, S_2)$, and represent a distribution over the set of edges $G_{S_2}$ (in the case of product nodes, we define $S_2 = S_{21} \cup S_{22})$. Thus, they also define a distribution over causal effects, given by:
    $$\causaleffect{i}{j}^{(t)} = \sum_{\pi \in F(S_2 \setminus \{j\})} \weights_{i,\pi_1} \weights_{\pi_{|\pi|}, j} \prod_{i=1}^{|\pi|-1} \weights_{\pi_i, \pi_{i+1}}$$
    which is defined for any distinct $i \in S_1 \cup S_2$, $j \in S_2$. Notice that this only counts paths which immediately enter (and stay in) $S_2$; thus all edges are in $G_{S_2}$.
    
    By taking the expectation, we can similarly define Bayesian averaged causal effects for node $t$:
    $$\bayesiance{i}{j}^{(t)} \triangleq \mathbb{E}_{G_{S_2} \sim q_{\spnweight}^{(t)}(G_{S_2})} [\mathbb{E}_{\weights \sim q(\weights_{S_2}|G_{S_2})} [\causaleffect{i}{j}^{(t)}(\weights_{S_2})] ] $$
    
    Given this, we now show how it is possible to decompose the computation of $\bayesiance{i}{j}$ according to the structure of the SPN. 
    
    If $t$ is a sum node, with children nodes $t_1, .. , t_k$ and corresponding weights $\spnweight^{(t)}_{1}, ... \spnweight^{(t)}_{C}$ we simply have that:
    \begin{align*}
        &\bayesiance{i}{j}^{(t)} \triangleq \mathbb{E}_{G_{S_2} \sim q_{\spnweight}^{(t)}(G_{S_2})} [\mathbb{E}_{\weights \sim q(\weights_{S_2}|G_{S_2})} [\causaleffect{i}{j}(\weights_{S_2})] ] \\
        &= \sum_{c = 1, ..., C} \spnweight^{(t)}_{c} \mathbb{E}_{G_{S_2} \sim q_{\spnweight}^{(t_c)}(G_{S_2})} [\mathbb{E}_{\weights \sim q(\weights_{S_2}|G_{S_2})} [\causaleffect{i}{j}(\weights_{S_2})] ] \\
        &= \sum_{c = 1, ..., C} \spnweight^{(t)}_{c} \bayesiance{i}{j}^{(t_c)}
    \end{align*}
    where we have used linearity of expectations to bring the sum outside.
    
    If $t$ is instead a product node, then it has two children $t_1, t_2$, which are associated with variable subsets $(S_1, S_{21})$, $(S_1 \cup S_{21}, S_{22})$, respectively. We now consider three separate cases, depending on where $i \in S_1 \cup S_2, j \in S_2$ are located within the subsets.
    \begin{itemize}
        \item If $i \in S_{22}$, $j \in S_{21}$, then $\bayesiance{i}{j}^{(t)} = 0$ since by construction edges (and by extension paths) from $S_{22}$ to $S_{21}$ are disallowed. 
        \item If $i \in S_1 \cup S_{21}$ and $j \in S_{21}$, or alternatively $i \in S_{22}$ and $j \in S_{22}$, then notice that all paths between $i, j$ must stay within $S_{21}$ or $S_{22}$ respectively, since there are no edges from $S_{22}$ to $S_{21}$. Thus, we have that $\causaleffect{i}{j}^{(t)} = \causaleffect{i}{j}^{(t_1)}$ or $\causaleffect{i}{j}^{(t_2)}$ (respectively) and 
        $$\bayesiance{i}{j}^{(t)} = \bayesiance{i}{j}^{(t_1)} \text{ or } \bayesiance{i}{j}^{(t_2)}$$
        \item In the final case, $i \in S_1 \cup S_{21}$ while $j \in S_{22}$. Here we must consider all possible paths between $i$ and $j$. To do so, we will condition on the last variable in $S_1 \cup S_{21}$ (``exit-point'') $k$ along a path. Then we have:
        \begin{align*}
        &\causaleffect{i}{j}^{(t)} = \sum_{\pi \in F(S_2 \setminus \{j\})} \weights_{i,\pi_1} \weights_{\pi_{|\pi|}, j} \prod_{i=1}^{|\pi|-1} \weights_{\pi_i, \pi_{i+1}} \\
        &= \sum_{k \in F(S_{21})}  \\
        & \left(\sum_{\pi \in F(S_{21} \setminus \{k\})}  \weights_{i,\pi_1} \weights_{\pi_{|\pi|}, k} \prod_{i=1}^{|\pi|-1} \weights_{\pi_i, \pi_{i+1}} \right) \\
        & \left(\sum_{\pi \in F(S_{22} \setminus \{j\})} \weights_{k,\pi_1} \weights_{\pi_{|\pi|}, j} \prod_{i=1}^{|\pi|-1} \weights_{\pi_i, \pi_{i+1}} \right) \\
        &= \sum_{k \in F(S_{21})} \causaleffect{i}{k}^{(t_1)} \causaleffect{k}{j}^{(t_2)}
        \end{align*}
        The last equality follows as the two summations are precisely the causal effects $i \to k$ and $k \to j$ for $t_1, t_2$, respectively, which correspond to variable subsets $(S_1, S_{21})$ and $(S_1 \cup S_{21}, S_{22})$.
        Now, by linearity of expectations, and the independence of $\causaleffect{i}{k}^{(t_1)}, \causaleffect{k}{j}^{(t_2)}$, this gives the matrix multiplication:
        $$\bayesiance{i}{j}^{(t)} =  \sum_{k \in F(S_{21} \setminus \{i\})} \bayesiance{i}{k}^{(t_1)} \bayesiance{k}{j}^{(t_2)}$$
    \end{itemize}
    Finally, we consider the leaf nodes $t$ of the SPN, where $|S_2| = 1$ (say, $S_2 = \{j\}$). In such cases, the causal effect reduces to $a_{ij}^{(t)} = B_{i,j}$, and the expectation is given by:
    $$\bayesiance{i}{j}^{(t)} = \mathbb{E}_{G_{j} \sim q_{\spnweight}^{(t)}(G_{j})} [\mathbb{E}_{\weights_j \sim q(\weights_{j}|G_{j})} [\weights_{i, j}] ] $$
    
    Given the graph column $G_j$, the distribution of $\weights_j$ is given by a multivariate $t$-distribution \citep{Viinikka20Scalable}, and so the inner expectation can be computed exactly for a given $G_j$. The outer expectation can be approximated using sampling from the leaf distribution. Though this involves sampling, the crucial aspect of our method is that the expectation through the OrderSPN (and thus through different orders) is exact, unlike \textit{Beeps} \citep{Viinikka20Scalable}, which computes causal effects using sampled DAGs.
    
    At each node $t$ corresponding to variable subsets $(S_1, S_2)$, we must maintain an array $BCE(i, j)^{(t)}$ for $i \in S_1 \cup S_2, j \in S_2$, i.e., of size $(|S_1| + |S_2|) \times |S_2| < d^2$. Computations at any node $t$ take linear time in the number of children (outgoing edges) of the node, except for the matrix multiplication at product nodes, which takes $(|S_1| + |S_{21}|) \times |S_{21}| \times |S_{22}| < d^3$ time. Thus, the overall space and time complexity is $O(d^2M)$ and $O(d^3M)$ respectively.
    
\end{proof}

% \propELBO*

% \begin{proof}
%     The proof here mostly follows \cite{Shih20Vi}, who show that the expectations of log-polynomial densities and entropy can be computed efficienty for a deterministic SPN. 
    
%     Recall that order modular distributions can be written as 
% \end{proof}

\section{OrderSPN Structure Learning Oracles}

In this section we elaborate further the oracles $\oracle$ used for generating the structure of the OrderSPN in Section \ref{sec:structurelearning}. As previously defined, the $\oracle$ takes as input some data $\dataset$, disjoint sets $S_1, S_2$, and a number of samples $K$, and returns $K$ partitions $(S_{21, i}, S_{22, i})$ of $S_2$. The goal of the oracle is to maximize coverage of the posterior distribution, i.e. the posterior mass of orders consistent with at least one of the sampled partitions. Solving such a problem exactly is clearly intractable; thus we would like heuristic methods which can obtain good coverage. 
%Since the oracle is called many times for different decompositions throughout the OrderSPN, it is to be expected that the choice of oracle will have a significant effect on the performance of the OrderSPN.

A possible oracle would simply be to take $K$ random partitions of $S_2$. However, this does not make efficient usage of the capacity of the OrderSPN. Thus, we consider adapting other Bayesian structure learning methods to take the role of the oracle. This can be done by sampling DAGs from the method; each such DAG naturally induces an order over the variables $S_2$, and thus a partition. Intuitively, we utilize their ability to find promising areas of the space of orders and DAGs to choose a better structure for our SPN. 

The key practical challenge is that, is in contrast to the typical use case, we are not just interested in learning a DAG over a set $S_2$, but also want to allow the variables in $S_2$ to have parents from some disjoint set $S_1$. This will require adaptations specific to the particular method chosen. In the rest of this section, we provide brief descriptions of how this can be done for \dibs{} and \gadget{}.

\dibs{} is a Bayesian structure learning approach based on particle variational inference \cite{Liu16Svgd}. In particular, in the marginal form, it assumes the following latent-variable generative model:
$$p(Z, G, \dataset) = p(Z) p(G|Z) p(\dataset|G)$$
where $Z = [U, V]$ with $U, V \in \mathbb{R}^{k \times d}$ (for some $k < d$) is a latent variable, generating the graph $G \in \{0, 1\}^{d \times d}$ and $\dataset$ is the dataset. In particular, the distribution for the graph takes the form:
\begin{align*}
    p(G|Z) = \prod_{i, j} p(G_{ij}|Z) 
    = \prod_{i, j} \sigma(\bm{u}_{i}^T \bm{v}_{j})
\end{align*}
Now suppose that we want to learn a DAG over $S_2$, where variables can additionally have parents in $S_1$. In this case, the natural generative model is to simply restrict to the components of the graph which are being modelled:
\begin{align*}
    p_{S_1, S_2}(G|Z) = \prod_{i \in {S_1 \cup S_2}} \prod_{j \in S_2} \sigma(\bm{u}_{i}^T \bm{v}_{j})
\end{align*}
The marginal likelihood $p(\dataset|\graph)$ is modular, so we can additionally restrict the likelihood to only concern the likelihood of $S_2$:
$$p_{S_2}(\dataset|\graph) = \prod_{j \in S_2} p(\dataset_j|\graph_j)$$

With these modifications, we have a valid generative model for any $(S_1, S_2)$, to which the \dibs{} particle variational inference scheme can be applied with no further changes, giving us an oracle.

\gadget{} is a MCMC method which samples over the space of \textit{ordered partitions} of the set of variables (note this is distinct from the 2-partitions we use in OrderSPNs). Informally speaking, the ordered partition represents a partial ordering of the variables, where a variable must have a parent from the partition directly preceding its partition. For example, for $d = 6$, a partition might be $(\{3, 4, 5\}, \{1\}, \{2, 6\})$, where variable $1$ must have one of $3, 4, 5$ as parent (but not any of $2, 6$). A $k-$partition $R$ is scored using a modular score:
$$\pi(R) = \prod_{t = 1}^k \prod_{j \in R_t} \tau_j(\cup_{i = 1}^{t-1} {R_i}, R_{t - 1}) $$
where $\tau_j(U, T)$ is the summed score (posterior probability) that variable $j$ has all parents contained in the set $U$, and at least one parent in the set $T$.

As mentioned in Appendix \ref{apx:singlenode}, \gadget{} \citep{Viinikka20Scalable} uses a similar type of precomputation to that used in \trust{} to precompute the functions $\tau_j(U, T)$, where a candidate parent set $C_j$ of each variable is chosen in advance (using a heuristic) so that we actually compute $\tau_j(U \cap C_j, T \cap C_j)$. 

Now, suppose we are given sets $(S_1, S_2)$, and as usual seek to learn DAGs over $S_2$ which additionally have parents from $S_1$. This can be achieved by simply restricting the MCMC to only learn ordered partitions over $S_2$, while also allowing the parent sets of variables $S_2$ to be contained in $S_1 \cup S_2$. In particular, if we have precomputed $\tau_j(U \cap C_j, T \cap C_j)$ for all $j$ and $U, T \subseteq \{1, ..., d\}$, this includes all of the necessary scores $\tau_j(U \cap C_j, T \cap C_j)$ for all $j \in S_2$ and $U, T \subseteq S_1 \cup S_2$, for any restriction $(S_1, S_2)$. 

The MCMC proceeds as if it were over a $|S_2|$ dimensional problem, over the set of variables $S_2$, but with modified scores involving $S_1$ as above, thus providing an oracle for TRUST.

% In practice, we adapt existing Bayesian structure learning methods to provide this oracle. As a result, the coverage of our SPN depends on the method chosen. 

% For DiBS, we adapt the method by 

\section{Parameter Learning and Tractable ELBO Computation} \label{apx:vi}

In this section, we provide further details on the variational inference scheme used to learn parameters of the OrderSPN. First, we provide a proof of Proposition \ref{prop:elbo}, which is based on Theorem 1 from \citep{Shih20Vi}. 

\propELBO*

\begin{proof}
We assume an order-modular distribution over the form $\tilde{p}(\order, G) = \prod_i p_{G_i}(G_i) \mathds{1}_{G \models \order}$. Define $\tilde{p}_{S_2}(\order_{S_2}, G_{S_2}) \triangleq \prod_{i \in S_2} p_{G_i}(G_i) \mathds{1}_{G_{S_2} \models \order_{S_2}}$ for any $S_2 \subseteq \bnvariables$. For any node $N$ in the OrderSPN with scope $(\order_{S_2}, G_{S_2})$, we will define the following quantity, which is the evidence lower-bound when using the distribution $N(\order_{S_2}, G_{S_2})$ to approximate $\tilde{p}_{S_2}(\order_{S_2}, G_{S_2})$:
$$ELBO(N) = \mathbb{E}_{N}[\tilde{p}_{S_2}(\order_{S_2}, G_{S_2})] + H(N(\order_{S_2}, G_{S_2}))$$

We now show that the ELBO for $q_\spnweight$ can be computed efficiently (i.e. in linear time in the size of the SPN) as a function of the ELBO of the leaf node distributions.

Let $T$ be a sum node associated with $(S_1, S_2)$, with children $C_1, ..., C_K$ and corresponding weights $\spnweight_{1}, ..., \spnweight_{K}$. We can write the expectation of $\tilde{p}_{S_2}$ and entropy in terms of corresponding quantities of the child distributions:
\begin{align*}
    \mathbb{E}_{T}[\log \tilde{p}_{S_2}(\order_{S_2}, G_{S_2})] =  \sum_{i = 1}^{K} \spnweight_i \mathbb{E}_{C_i}[\log \tilde{p}_{S_2}(\order_{S_2}, G_{S_2})] 
\end{align*}
\begin{align*}
    H(T(\order_{S_2}&, G_{S_2})) = -\mathbb{E}_{T}[\log T(\order_{S_2}, G_{S_2})] \\
    &= - \sum_{i = 1}^{K} \spnweight_i \mathbb{E}_{C_i}[\sum_{j = 1}^{K} \log \spnweight_j C_j(\order_{S_2}, G_{S_2})] \\
    &= - \sum_{i = 1}^{K} \spnweight_i \mathbb{E}_{C_i}[\log \spnweight_i C_i(\order_{S_2}, G_{S_2})] \\
    &= - \sum_{i = 1}^{K} \spnweight_i \log \spnweight_i + \sum_{i = 1}^{K} \spnweight_i \mathbb{E}_{C_i}[-C_i(\order_{S_2}, G_{S_2})] \\
    &= - \sum_{i = 1}^{K} \spnweight_i \log \spnweight_i + \sum_{i = 1}^{K} \spnweight_i H(C_i(\order_{S_2}, G_{S_2})) \\
\end{align*}
\begin{align*}
    &ELBO(T) = \mathbb{E}_{T}[\log \tilde{p}_{S_2}(\order_{S_2}, G_{S_2})] + H(T(\order_{S_2}, G_{S_2})) \\
    &= - \sum_{i = 1}^{K} \spnweight_i \log \spnweight_i \\
    & \;\;\;\; + \sum_{i = 1}^{K} \spnweight_i \left[\mathbb{E}_{C_i}[\log \tilde{p}_{S_2}(\order_{S_2}, G_{S_2})] + H(C_i(\order_{S_2}, G_{S_2}))\right] \\
    &= - \sum_{i = 1}^{K} \spnweight_i \log \spnweight_i + \sum_{i = 1}^{K} \spnweight_i ELBO(C_i)
\end{align*}

In other words, the expectation decomposes as a weighted sum over expectations with respect to the child distributions, and the entropy decomposes as a sum of the entropy of the sum-node weights, and a weighted sum over entropies with respect to the child distributions. Note that the third equality in the derivation of the entropy decomposition holds only due to the fact that OrderSPNs are deterministic; this means that the children $C_i$ have disjoint supports, and thus $\mathbb{E}_{C_i}[C_j(\order, G)] = 0$ for all $i \neq j$. Together, we have that the ELBO of a sum-node can be expressed in terms of the ELBO of its children.

Let $P$ be a product node, associated with $((S_1, S_{21}), (S_{21}, S_{22}))$, with children $C_1, C_2$. $P$ expresses a distribution over $(\order_{S_2}, G_{S_2})$, where $S_2 = S_{21} \cup S_{22}$. Then we have that:
\begin{align*}
&\mathbb{E}_{P}[\log \tilde{p}_{S_2}(\order_{S_2}, G_{S_2})] =  \mathbb{E}_{P}[\log \prod_{i \in S_2} p_{G_i}(G_i) \mathds{1}_{G_{S_2} \models \order_{S_2}}]\\
&= \mathbb{E}_{P}[\log \prod_{i \in S_{21}} p_{G_i}(G_i) \mathds{1}_{G_{S_{21}} \models \order_{S_{21}}} \\
& \;\;\;\;\;\;\;\;\;\; \log \prod_{i \in S_{22}} p_{G_i}(G_i) \mathds{1}_{G_{S_{22}} \models \order_{S_{22}}}] \\
&= \mathbb{E}_{P}[\log \tilde{p}_{S_{21}}(\order_{S_{21}}, G_{S_{21}})] + \mathbb{E}_{P}[\log \tilde{p}_{S_{22}}(\order_{S_{22}}, G_{S_{22}})]
\end{align*}
\begin{align*}
    &H(P(\order_{S_{2}}, G_{S_{2}})) = -\mathbb{E}_{P}[\log P(\order_{S_{2}}, G_{S_{2}})] \\
    &= -\mathbb{E}_{P}[\log C_1(\order_{S_{21}}, G_{S_{21}}) + \log C_2(\order_{S_{22}}, G_{S_{22}})]\\
    &= -\mathbb{E}_{C_1}[\log C_1(\order_{S_{21}}, G_{S_{21}})] - \mathbb{E}_{C_2}[\log  C_2(\order_{S_{22}}, G_{S_{22}})]\\
    &= H(C_1(\order_{S_{21}}, G_{S_{21}})) + H(C_2(\order_{S_{22}}, G_{S_{22}})) 
\end{align*}
\begin{align*}
    &ELBO(P) = \mathbb{E}_{P}[\log \tilde{p}_{S_2}(\order_{S_{2}}, G_{S_{2}})]  + H(P(\order_{S_{2}}, G_{S_{2}})) \\
    &= \mathbb{E}_{P}[\log \tilde{p}_{S_{21}}(\order_{S_{21}}, G_{S_{21}})] + H(C_1(\order_{S_{21}}, G_{S_{21}})) \\
    &\;\;\;\; + \mathbb{E}_{P}[\log \tilde{p}_{S_{22}}(\order_{S_{22}}, G_{S_{22}})] + H(C_2(\order_{S_{22}}, G_{S_{22}})) \\
    &= ELBO(C_1) + ELBO(C_2)
\end{align*}
%\martacomment{check, fix + +, break}
This follows from decomposability, which ensures that the child distributions are over disjoint sets of variables (and are thus independent).

By recursively applying the above equalities, we can express the ELBO for the overall OrderSPN $q_\phi$ in terms of the SPN weights $\phi$ and ELBO for the leaf node distributions. Since each equality involves a sum/product over the children of the node (i.e., the outgoing edges), the overall computation takes linear time in the size (number of edges) of the SPN.
\end{proof}

\subsection{ELBO for Leaf Node Distributions}
In the above Proposition, we have not mentioned how to compute the ELBO for the leaf node distributions. For a leaf node $L$ associated with $(S_1, i)$, which is a distribution $L(G_i)$ over the parents of variable $i$, we have that:
\begin{align}\label{align:leaf}
    ELBO(L) &= \mathbb{E}_{L}[\log \tilde{p}(\order_{\{i\}}, G_i)] + H(L(\order_{\{i\}}, G_i)) \nonumber \\
    &= \mathbb{E}_{L}[\log p_{G_i}(G_i)] + H(L(G_i))
\end{align}
Recall that, for OrderSPNs, it is required that $L(G_i)$ has support only over $G_i \subseteq S_1$. In the main paper, we chose to set $L(G_i) \propto p_{G_i}(G_i) \mathds{1}_{G_i \subseteq S_1}$. We now provide justification for this choice:
\begin{proposition} \label{prop:leafjustification}
    $L(G_i) \propto p_{G_i}(G_i) \mathds{1}_{G_i \subseteq S_1}$ maximizes (\ref{align:leaf}) subject to the support condition.
\end{proposition}

\begin{proof}
The ELBO for a leaf distribution (\ref{align:leaf}) can be written as:
\begin{align*}
    ELBO(L) &= \mathbb{E}_{L}[\log p_{G_i}(G_i)] + H(L(G_i))) \\
    &= \mathbb{E}_{L}[\log p_{G_i}(G_i)] - \mathbb{E}_{L}[\log L(G_i)]\\
    &= -KL(L||p_{G_i})
\end{align*}
where $KL$ is the KL-divergence. Thus, to maximize the ELBO, we need to minimize this KL-divergence. Let $C = \sum_{G_i \subseteq S_1} p_{G_i}(G_i)$. Assuming $L$ satisfies the support condition, this can be written as:
\begin{align*}
    KL(L||p_{G_i}) &= \mathbb{E}_{L} \left[\log \frac{L(G_i)}{p_{G_i}(G_i)}\right] \\
    &= \mathbb{E}_{L} \left[\log \frac{L(G_i)}{p_{G_i}(G_i)}\right] \\
    &= \mathbb{E}_{L} \left[\log \frac{L(G_i)}{p_{G_i}(G_i)/C}\right]  - \log C\\
\end{align*}
This KL-divergence is minimized by $L(G_i) \propto p_{G_i}(G_i) \mathds{1}_{G_i \subseteq S_1}$, as required. In this case, the ELBO is given by:
\begin{align*}
    ELBO(L) &= \log C - KL(L||\frac{p_{G_i}\mathds{1}_{G_i \subseteq S_1}}{C}) \\
    &= \log C
\end{align*}
\end{proof}
We see that, with this choice of $L$, the ELBO is a constant $\log C$ that we can precompute using the methods for computation of leaf distribution described in Appendix \ref{apx:singlenode}. Thus, the computation of ELBO for leaf distributions can be done in an $O(1)$ lookup, and the overall ELBO computation is linear in the size of the OrderSPN (in particular, independent of the dimension).

% \subsubsection{Interventional Distributions}

% Notice that the previous results were derived for any order-modular distribution represented as $\tilde{p}$. In particular, this allows us to select other types of distributions other than the observational marginal likelihoods. For instance, suppose we were interested in learning on interventional data from a known intervention, for instance, 

%    &= \log Z - KL(L||p)

\section{Experimental Details}

\paragraph{Bayesian network hyperparameters} In our experiments, we consider linear Gaussian Bayesian networks, and generate Erdos-Renyi random structures, with expected numbers of edges given by $2d$. We generate data using fixed observation noise $\sigma^2 = 0.1$, and edge weights drawn independently from $\mathcal{N}(0, 1)$. 

\paragraph{Posterior setup} We use the fair prior over graph structures \citep{Eggeling19Fairprior}, where the prior probability of a mechanism having $k$ edges is proportional to the inverse of the number of different parents sets of size $k$. In addition, we use the BGe marginal likelihood \citep{Kuipers14Bge} with hyperparameters $\alpha_\mu = 1$, $\alpha_w = d + 2$, and $T = \frac{1}{2} I$ where $I$ is the $d \times d$ identity matrix.

\paragraph{Implementation details} Our implementations of \dibs{} and \gadget{} are based on the reference implementations with the default settings of hyperparameters. In particular, we ran \dibs{} with $N=30$ particles and 3000 epochs using the marginal inference method, while \gadget{} was run using 16 coupled chains and for 320000 MCMC iterations, extracting $N = 10000$ samples. 

Our implementation of TRUST uses the PyTorch framework to tensorize passes through the SPN, following the regular OrderSPN structure described in the main paper. In the $d=16,32$ cases, we used expansion factors of $\bm{K} = [64, 16, 6, 2], [32, 8, 2, 6, 2]$ respectively; these were chosen empirically to approximately match oracle computation across layers. Parameter learning in the SPN was performed by optimizing the ELBO objective using the Adam optimizer with learning rate $0.1$ and for $700$ iterations. Operations in the circuit are performed in log-space for numerical stability.

\paragraph{Inference Queries} We perform inference for DiBS and Gadget by applying the appropriate calculation over the sample (for instance, the marginal probability of an edge $G_ij$ is simply the proportion of sampled DAGs in which it appears), while for \trust{}, we perform inference directly on the OrderSPN using the queries described in the paper when this is possible, and by sampling otherwise (e.g. E-SHD).

\section{Ablation Study on OrderSPN Learning} \label{sec:ablation}

In Section \ref{sec:learning}, we proposed to use a two-step procedure for learning OrderSPNs, in which we (i) propose a structure for the OrderSPN using an oracle method; and (ii) further learn the parameters of the OrderSPN via variational inference. We now perform an ablation study to examine the each of these steps and their impact on performance.

We evaluate five different methods:
\begin{itemize}
    \item \textbf{Random} In this case, instead of using an oracle method $\oracle$ to split $S_2$ into a partition $(S_{21, i}, S_{22, i}$, we instead perform this split \textit{randomly} throughout the OrderSPN. We also do not perform any parameter learning, instead setting the parameters at each sum-node in the OrderSPN to be equal (e.g. if a sum-node has 4 children, we set each parameter to 0.25).
    \item \textbf{Parameter Only} We randomly propose the structure as above, but do perform parameter learning using VI.
    \item \textbf{Structure Only} We do perform structure learning using \gadget{} as an oracle, but do not learn parameters.
    \item \textbf{Gadget} As in the main paper.
    \item \textbf{TRUST-G} As in the main paper.
\end{itemize}

The first step of structure learning determines the support of the OrderSPN, i.e. the orders and DAGs to which it assigns positive probability, while the second step of parameter learning aims to optimize the fit to the posterior given the support constraints imposed by the first step. By randomizing one (or both) of these steps, we can see how this affects the approximation. 

The results are shown in Figure \ref{fig:ablation}. As expected, the fully random method performs by far the worst, on all metrics. Both performing parameter learning only and structure learning only provide significant improvements, but interestingly on different metrics. Structure learning only performs quite well on AUROC, while parameter learning only performs comparatively better on E-SHD and MLL (even outperforming \gadget{} on E-SHD). The performance of using parameter learning only is quite remarkable, given that the graphs covered by the OrderSPN were chosen \textit{at random}. We hypothesize that this can be attributed to the compactness and capacity of OrderSPNs as a representation; as a result, even the randomly chosen structure will contain some orders/DAGs which are close to the ground truth DAG. Nonetheless, adding structure learning as well, as in \trustg{}, does provide the best overall performance, and shows that both steps are important to obtain the best possible representation.

\begin{figure*}[htb]
    \centering
    \includegraphics[width=0.473\linewidth]{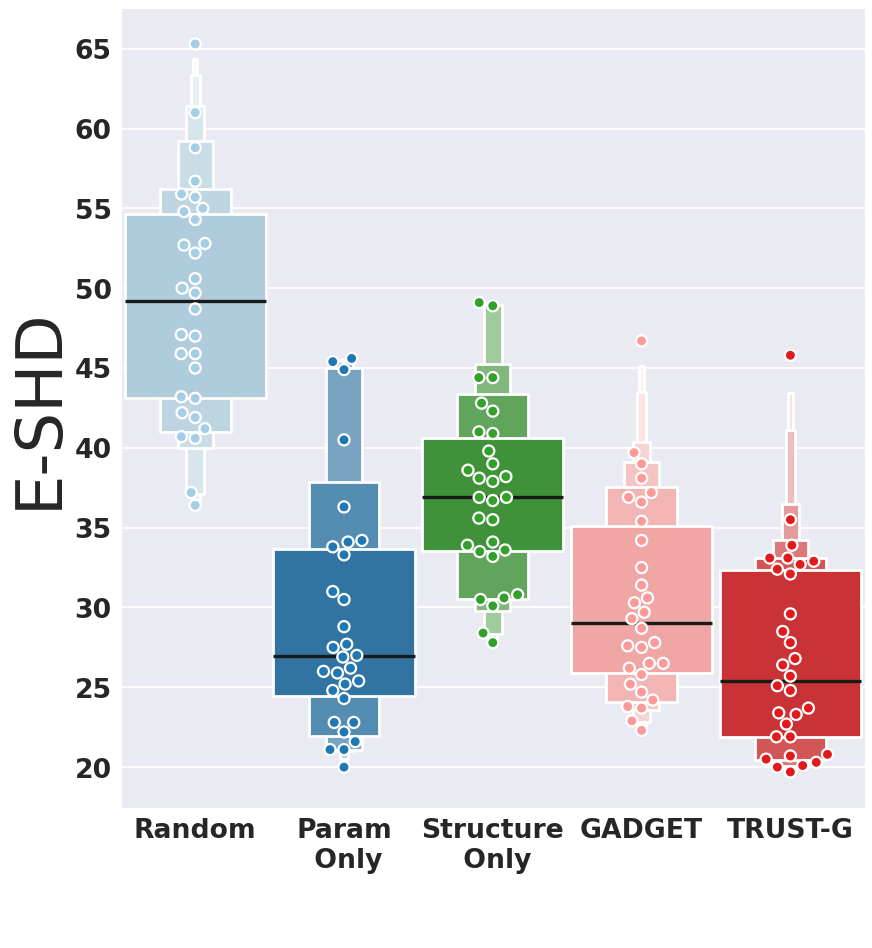}
    \includegraphics[width=0.49\linewidth]{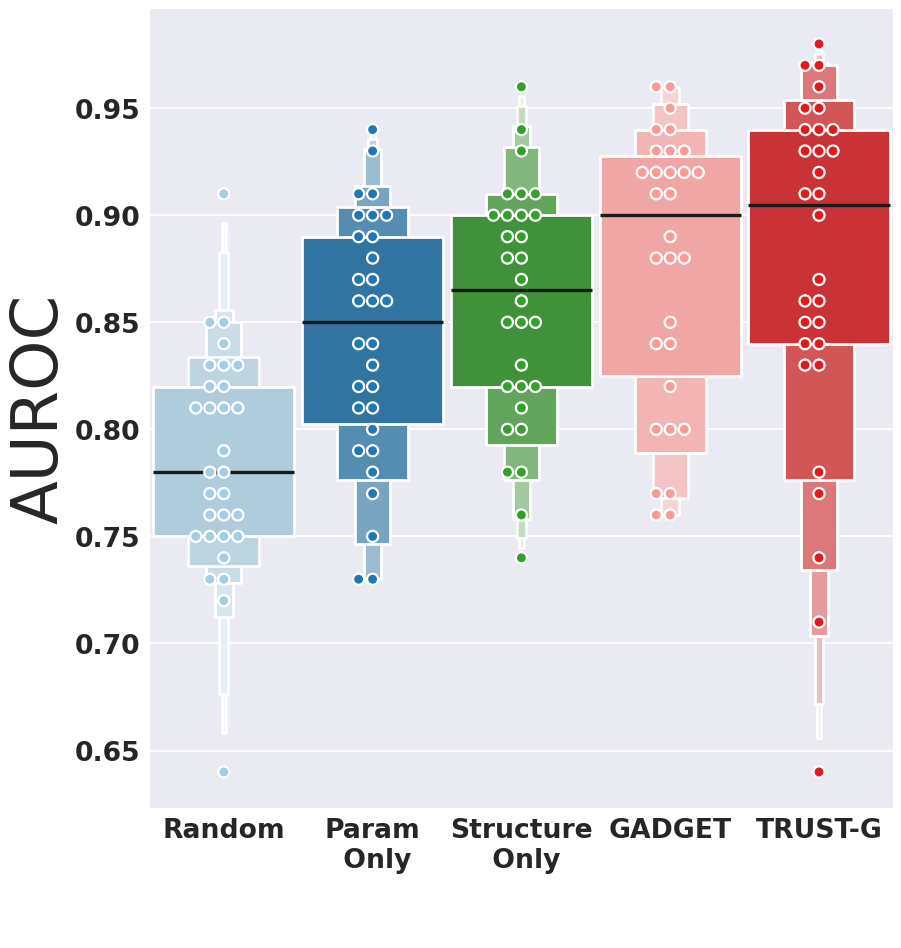}
    \includegraphics[width=0.49\linewidth]{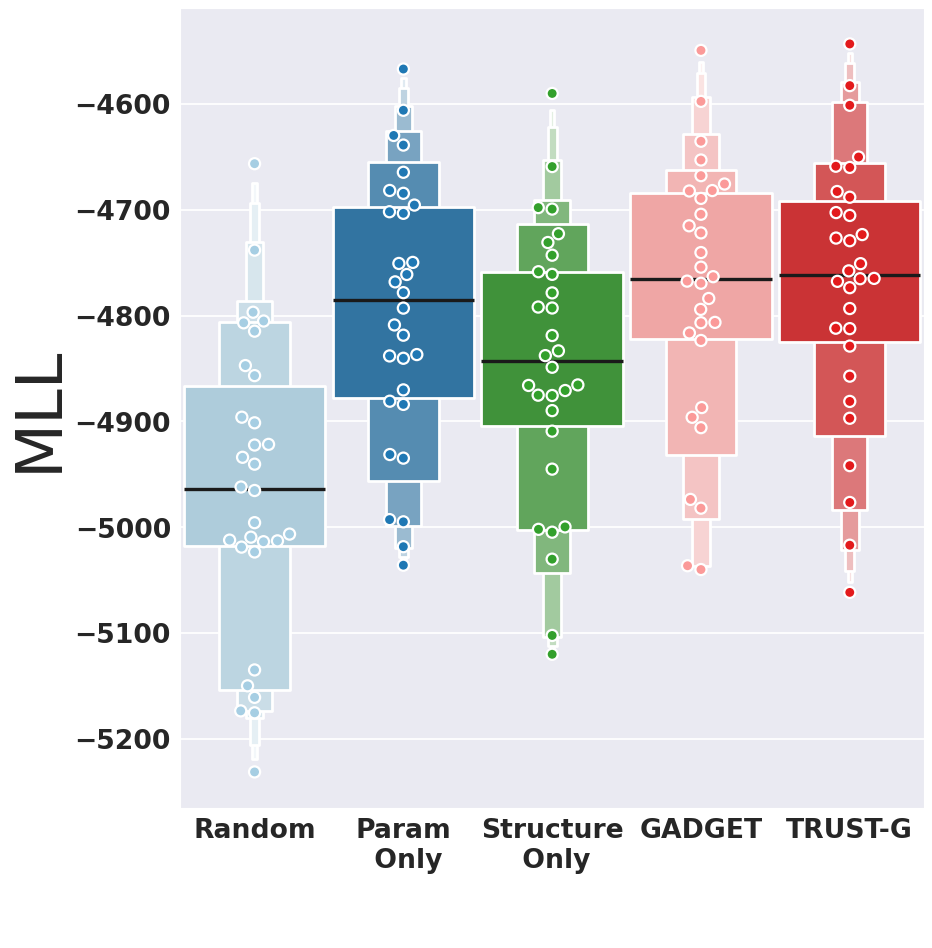}
    \includegraphics[width=0.473\linewidth]{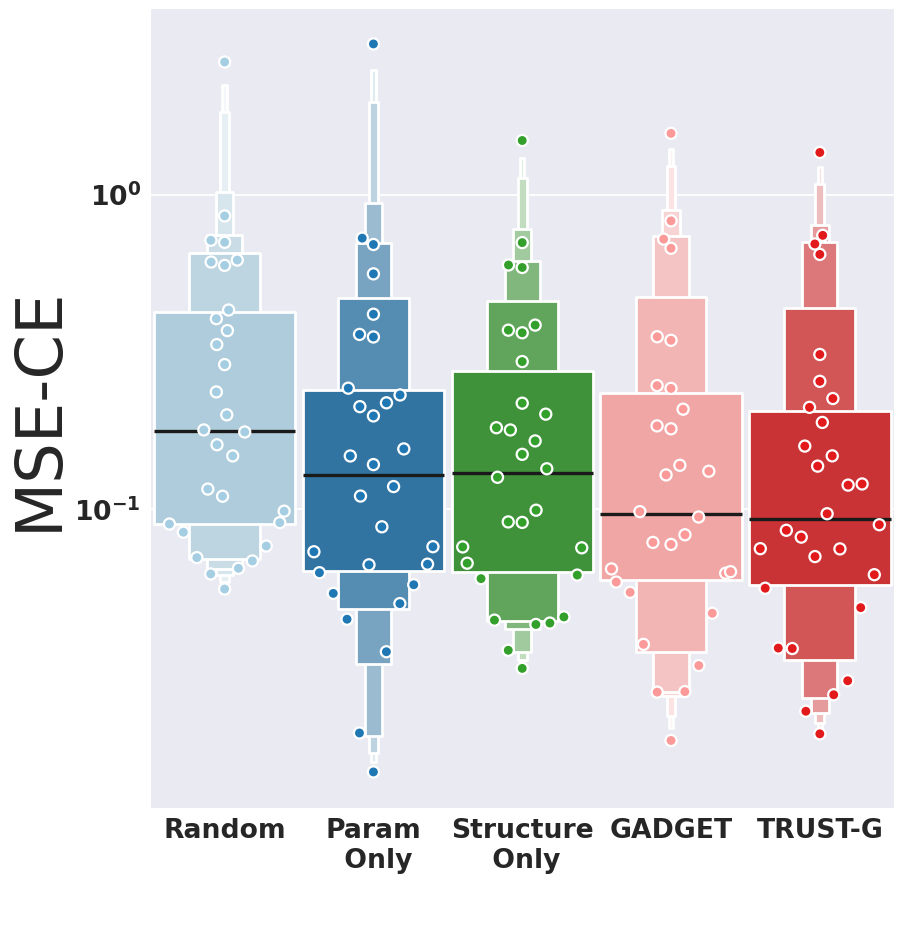}
    \caption{Ablation study evaluating performance of different variants of \trustg{} (and \gadget{}), for $d=16$.}
    \label{fig:ablation}
\end{figure*}

\end{document}